\documentclass[letterpaper]{article} 
\usepackage{aaai24}  
\usepackage{times}  
\usepackage{helvet}  
\usepackage{courier}  
\usepackage[hyphens]{url}  
\usepackage{graphicx} 
\usepackage{natbib}  
\usepackage{caption} 
\usepackage{algorithm}
\usepackage{algorithmic}

\usepackage{listings}
\DeclareCaptionStyle{ruled}{labelfont=normalfont,labelsep=colon,strut=off} 
\usepackage{algorithm}
\usepackage{algorithmic}
\usepackage{amsfonts} 
\usepackage{amsmath} 
\usepackage{subcaption}
\usepackage{amsthm}
\usepackage{array}
\usepackage{siunitx}
\usepackage[figuresleft]{rotating}
\usepackage{multirow}

\newcommand{\cL}{{\mathcal{L}}}

\newcommand{\cC}{{\mathcal{C}}}

\newcommand{\bC}{\textbf{C}}
\newcommand{\bR}{\textbf{R}}

\newcommand{\bc}{\textbf{c}}

\newcommand{\bbE}{\mathbb{E}}

\newcommand{\KL}{\textsc{KL}}

\newcommand{\showLegend}[6][0.25]{
\begin{minipage}{#1\textwidth}
    \centering
    \includegraphics[width=1.0\textwidth,trim={#2 #3 #4 #5},clip]{#6.pdf}
\end{minipage}
}
\newcommand{\showReturn}[7][0.25]{
\begin{minipage}{#1\textwidth}
    \caption*{\detokenize{#7}}
    \captionsetup{justification=centering}
    \centering
    \includegraphics[width=1.0\textwidth,trim={#2 #3 #4 #5},clip]{#6.pdf}
\end{minipage}
}
\newcommand{\showCost}[6][0.25]{
\begin{minipage}{#1\textwidth}
    \centering
    \includegraphics[width=1.0\textwidth,trim={#2 #3 #4 #5},clip]{#6.pdf}
\end{minipage}
}

\newcommand{\showEnv}[7][0.25]{
\begin{minipage}{#1\textwidth}
    \centering
    \includegraphics[width=1.0\textwidth,trim={#2 #3 #4 #5},clip]{#6.pdf}
    \caption*{\detokenize{#7}}
    \captionsetup{justification=centering}
\end{minipage}
}

\theoremstyle{plain}

\newtheorem{proposition}{Proposition}
\newtheorem{lemma}{Lemma}

\theoremstyle{definition}
\newtheorem{definition}{Definition}

\theoremstyle{remark}

\newcommand{\squishlist}{
 \begin{list}{$\bullet$}
  { \setlength{\itemsep}{0pt}
     \setlength{\parsep}{2pt}
     \setlength{\topsep}{2pt}
     \setlength{\partopsep}{0pt}
     \setlength{\leftmargin}{1.5em}
     \setlength{\labelwidth}{1em}
     \setlength{\labelsep}{0.5em} } }

\newcommand{\squishend}{
  \end{list}  }

\newif\ifnotes\notestrue
\usepackage{listings}
\DeclareCaptionStyle{ruled}{labelfont=normalfont,labelsep=colon,strut=off} 
\title{Imitate the Good and Avoid the Bad: An Incremental Approach to Safe Reinforcement Learning}
\author {
    Huy Hoang,
    Tien Mai,
     Pradeep Varakantham
}
\affiliations {
    School of Computing and Information Systems\\
    Singapore Management Univerisity, Singapore\\
    \{mhhoang, atmai, pradeepv\}@smu.edu.sg
}

\usepackage{bibentry}

\begin{document}

\maketitle

\begin{abstract}
A popular framework for enforcing safe actions in Reinforcement Learning (RL) is Constrained RL, where {\em trajectory} based constraints on expected cost (or other cost measures) are employed to enforce safety and more importantly these constraints are enforced while maximizing expected reward. Most recent approaches for solving Constrained RL convert the trajectory based cost constraint into a surrogate problem that can be solved using minor modifications to RL methods. A key drawback with such approaches is an over or underestimation of the cost constraint at each state. Therefore, we provide an approach that does not modify the trajectory based cost constraint and instead imitates ``good'' trajectories and avoids ``bad'' trajectories generated from incrementally improving policies. We employ an oracle that utilizes a reward threshold (which is varied with learning) and the overall cost constraint to label trajectories as ``good'' or ``bad''. A key advantage of our approach is that we are able to work from any starting policy or set of trajectories and improve on it. In an exhaustive set of experiments, we demonstrate that our approach is able to outperform top benchmark approaches for solving Constrained RL problems, with respect to expected cost, CVaR cost, or even unknown cost constraints. Code is available at:~\url{https://github.com/hmhuy0/SIM-RL}.

\end{abstract}

\section{Introduction}
Reinforcement learning (RL) is widely acknowledged as a powerful paradigm, thanks to its exceptional ability to learn and adapt by interacting with the environment. This adaptability has been demonstrated through numerous studies that highlight its practical applications across diverse domains. For example, reinforcement learning has been successfully employed in video games to achieve groundbreaking results~\cite{mnih2016asynchronous, firoiu2017beating}, robot manipulation tasks have been enhanced using this approach~\cite{hoang2023learning, kilinc2022reinforcement}, and even the field of healthcare has benefited from its potential~\cite{weng2017representation, raghu2017deep}. In light of the notable achievements of reinforcement learning, it is crucial to acknowledge the practical limitations that come with this approach when applied to real-world situations. The constraints of limited resources, budgetary restrictions, and safety concerns pose significant challenges in implementing reinforcement learning effectively. 

\noindent \textbf{\em Constrained RL: }
To address these challenges, Constrained Markov Decision Processes (CMDPs) have been developed as an extension of Markov Decision Processes (MDPs)~\cite{altman1999constrained}. CMDPs have emerged as a valuable framework for decision-making in various domains, as they allow for the optimization of objectives while ensuring the fulfillment of trajectory-based constraints over expected cost and other measures (e.g., CVaR). In order to tackle the challenges posed by these constraints, several Constrained RL algorithms have been proposed~\cite{yang2022constrained, zhang2020first}. State-of-the-art constrained RL approaches~\cite{satija2020constrained, chow2019lyapunov, achiam2017constrained} convert trajectory-based cost constraints into local cost constraints that can be solved easily while guaranteeing the enforcement of trajectory-based constraints. One potential issue with such local cost constraints in challenging constrained RL problems is the  estimation of cost value functions. Due to the difficulty involved in estimating the costs of partial (or full) trajectories, output policies can either be conservative or aggressive with regard to costs. 

In this work, we develop a novel principled framework that avoids the use of local cost constraints and, instead, focuses on directly solving the original constrained MDP problem, thereby avoiding cost estimation. Our innovation is rooted in the observation that, within the context of CMDP, from a given set of trajectories, it is easy to identify ``good'' trajectories that are feasible with respect to the cost constraints and offer high rewards. In contrast, ``bad'' trajectories would be identified as infeasible with respect to the cost constraints and/or yield low rewards. Subsequently, a policy that assigns high probabilities to good trajectories becomes a strong candidate for effectively addressing the CMDP problem. Hence, our approach to address CMDP involves learning a policy that replicates the actions of the good trajectories while steering clear of the bad ones. We do this by employing imitation learning, a framework designed to imitate an expert's policy based on their demonstrations.


\noindent \textbf{\em Imitation Learning:}
Imitation learning (IL) has been recognized as a compelling approach for making sequential decisions \citep{ng2000algorithms,abbeel2004apprenticeship}. In IL, a set of expert trajectories is provided, and the aim is to train a policy that replicates the behavior of the expert's policy. One of the simplest IL methods is Behavioral Cloning (BC), which mimics an expert's policy by maximizing the likelihood of the expert’s actions under the learned policy. 
BC  is simple to implement but it disregards environmental dynamics, 
making it unable to perform as well as an expert in unseen states~\cite{ross2011reduction}. 
To address this issue, Generative Adversarial Imitation Learning~\cite{ho2016generative} and Adversarial Inverse Reinforcement Learning~\cite{fu2017learning} were introduced. These methods use adversarial training to make the agent's behavior match the expert's occupancy distribution as estimated by their discriminator. However, the adversarial training often hinders the agent from achieving expert-level performance, especially in continuous settings. ValueDICE~\cite{kostrikov2019imitation} learns a value function based on the KL divergence of the learner and expert occupancy distributions and performs well in offline settings while still incorporating adversarial training. More recent methods like PWIL~\cite{dadashi2020primal} and IQ-learn~\cite{garg2021iq} use different statistical distances for occupancy distribution and successfully eliminate the need for adversarial training.

It is important to note that imitation within our context differs from the conventional IL approaches from the aforementioned works. Here, our approach not only involves mimicking the behavior of ``good'' demonstrations but also actively avoiding the bad ones. To the best of our knowledge, this marks the first time the concept of learning to avoid bad demonstrations is introduced within the realm of IL. Additionally, in a standard IL algorithm, the set of expert demonstrations is fixed beforehand. In contrast, in our context, the set of demonstrations is generated by a pre-trained or learning policy, thus allowing it to expand as training progresses. These factors collectively pave the way for the development of a novel IL algorithm that is well-suited to our specific context.

\noindent \textbf{\em Contrastive Learning:} Our framework 
is also related to the context of Contrastive Learning (CL). CL was first introduced by~\cite{bromley1993signature} with the Siamese architecture to create a mapping function for the inputs into a target space where two similar samples should be close while two different classes should be far away. 
There are several famous applications of CL in computer vision~\cite{noroozi2016unsupervised,he2019moco,grill2020bootstrap}, natural language processing~\cite{clark2020electra,gao2021simcse}, recommendation systems~\cite{zhou2021contrastive,xie2021adversarial}, and reinforcement learning~\cite{fu2021towards,laskin2020curl}. Our algorithm  shares a similar spirit with CL and also marks the  first time the idea of contrastive learning being applied in IL.






\noindent \textbf{\em Contributions:} 
We make the following contributions:
\begin{itemize}
    \item \textit{New framework for Constrained RL:} We propose a novel training framework for Constrained RL that incrementally improves an agent policy by imitating  ``good'' trajectories and avoiding  ``bad'' trajectories. The sets of ``good'' and ``bad'' trajectories are selected based on their accumulated rewards and costs and are updated as the policy is improved. 
    \item \textit{Theoretical insights:} We show that our way of imitating the good trajectories and avoiding the bad ones can be shown to ensure no deterioration in the output policy performance.
    \item \textit{New Learning algorithm}: We develop a non-adversarial \textit{imitate and avoid} algorithm that is able to imitate ``good'' trajectories and avoid ``bad'' trajectories. Due to the non-adversarial nature of the algorithm, it provides higher stability while being scalable.  
    \item \textit{Experimental results:} We provide an extensive experimental results section, where we demonstrate that our approach outperforms existing best approaches on all six different environments\footnote{Existing works have typically showed results only on the simplest environment, Safety Point Goal-v0.} within the highly challenging Safety-Gym benchmark.  Furthermore, we also provide results for expected cost,  CVaR cost, and \textit{unknown cost} settings.    
\end{itemize}

\section{Background}
We present a description of the Constrained MDP problem and some popular IL approaches. 
\subsection{Constrained Markov Decision Process}
The Markov Decision Process (MDP) described in~\cite{altman1999constrained} can be represented as $\mathcal{M} = \left\langle S, A, r, P, s_0 \right\rangle$. Here, $S$ denotes the set of states, $s_0$ represents the initial state set, $A$ is the set of actions, $r: S \times A \rightarrow \mathbb{R}$ defines the reward function for each state-action pair, and $P: S \times A \rightarrow S$ is the transition function.

By introducing an additional constraint set $\mathcal{C}=\left\langle d,c_{\text{max}} \right\rangle$ to the MDP, we can formulate a Constrained Markov Decision Process (CMDP). The constraint set includes a cost function $d: S \rightarrow \mathbb{R}$ and a maximum allowed accumulated cost $c_{\text{max}}$. The objective of the CMDP is to maximize the return while ensuring that the expected accumulated cost remains below the specified maximum. Mathematically, the objective function and constraint can be expressed as follows:
\begin{equation}\tag{\sf\small CMDP}
\begin{aligned}
\label{equ:cmdp}
    &\max_\pi \mathbb{E}\left[\sum_{t=0}^{\infty} {\gamma^t}r(s_t,a_t)|s_0,\pi\right] \\
    &
    \text{  s.t.  } \quad \mathbb{E}\left[\sum_{t=0}^{\infty}  \gamma^t d(s_t)|s_0,\pi\right]\leq c_{max}.
\end{aligned}
\end{equation}
where $\pi$ represents a policy, $\gamma$ is the discount factor, and the expectation is taken with respect to the initial state and the policy.
From now, to simplify the notion, we define $R({\tau})$ and $C(\tau)$ be the expectation of return and accumulated cost on trajectories $\tau$, i.e., $R(\tau)=\sum_{(s_t,a_t) \in \tau} \gamma^t r(s_t,a_t)$, $C(\tau)=\sum_{s_t \in \tau} \gamma^t d(s_t)$.

\subsection{Imitation Learning}
\paragraph{Behavioral Cloning.}
In BC, the objective is to maximize the likelihood of the demonstrations.
\begin{equation}
\label{BC}\tag{\sf BC}
\max_{\pi} \bbE_{\tau \sim \pi^E} \Big[ \sum_{(s,a) \in \tau} \ln (\pi(a|s))\Big]    
\end{equation}
BC has a strong theoretical foundation
but ignores environmental dynamics and 
only works with offline learning, requiring a  huge number of samples to achieve a desired performance~\cite{ross2011reduction}.

\paragraph{Distribution matching.}
A popular and useful approach for IL is based on state-action distribution matching. Specifically, let $\rho^\pi(s, a)$ be the occupancy measure of visiting state $s$ and taking action $a$, under policy $\pi$. Let $\rho^{\pi^E}$ the state-action distribution given by expert policy $\pi^E$. The distribution matching approach proposes to learn $\pi$ to minimize the discrepancy between $\rho^\pi$ and $\rho^{\pi^E}$ such as KL-divergence:
{\small\begin{equation}
\label{eq:dist_matching}
    \min_{\pi} KL\left(\rho^{\pi} ||\rho^{\pi^E} \right) = \min_{\pi}\left\{ \bbE_{(s,a)\sim \rho^\pi} \left[\ln \frac{\rho^{\pi^E}(s,a)}{\rho^{\pi}(s,a)}\frac{}{}\right] \right\}
\end{equation}}
Approaches based on distributional matching include some state-of-the-art IL algorithms such as adversarial IL \citep{ho2016generative,fu2017learning} or IQ-learn \citep{garg2021iq}.


\section{Self-Imitation Learning Approach}

Before describing our learning approach, we define ``Good'' and ``Bad'' trajectories:
\begin{definition}
\label{def:good-bad}
\textit{A trajectory, $\tau$ is a good trajectory if: $R(\tau) \geq R_G$ and $C(\tau) \leq c_{max}$. On the other, a trajectory, $\tau$ is a bad trajectory if $R(\tau) < R_B$ or $C(\tau) > c_{max}$.}
\end{definition}
Here, $R_G$ and $R_B$ represent some predefined\footnote{We provide an in-depth analysis on the selection of these hyperparameters and changing them during the learning for certain problems. } thresholds for selecting good and bad trajectories, respectively. We denote $\Omega^G$ and $\Omega^B$ as the set of good and bad trajectories respectively.
\subsection{Learning from Good and Bad Demonstrations}
Our aim is to train an RL agent to imitate the good behavior from a  set of good demonstrations (trajectories) and avoid the bad demonstrations. In other words, we try to mimic the good part of the pre-trained policy and avoid the bad part. 

\noindent\textbf{Behavior Cloning GB:} When using a Behavior Cloning, BC type approach to achieve the above objective,  the aim is to maximize the likelihood of the good set while minimizing the likelihood of the bad one. The training objective can be written as:
\begin{multline}
    \label{BC-good-bad}\tag{\textsf{BC-GB}}
    \max_{\pi}\left\{ \lambda \bbE_{\substack{\tau \sim \pi^0\\\tau \in \Omega^G}} \Big[ \sum_{(s,a) \in \tau} \ln (\pi(a|s))\Big] \right. \\
    \left. - (1-\lambda) \bbE_{\substack{\tau \sim \pi^0\\\tau \in \Omega^B}} \Big[ \sum_{(s,a) \in \tau} \phi(\ln (\pi(a|s)))\Big] \right\}
\end{multline}
where $\phi(\cdot)$ is a monotone regularizer mapping $(-\infty,0)$ to a finite interval,  and 
$\lambda\in [0,1]$ is a parameter capturing the impact of each good or bad set on the objective function, and $\pi^0$ is a starting policy that we want to improve upon. We use $\pi^0$ instead of $\pi^E$ as the starting policy is not necessarily an expert one. 
If  $\lambda = 1$, then we only learn from good demonstrations and ignore bad ones, and $\lambda = 0$ otherwise. We incorporate the regularization term $\phi(\cdot)$ in this context to address a critical concern. Without this regularization, the maximization process could drive the value of $\ln(\pi(a|s))$ in the second term of equation \eqref{BC-good-bad} towards negative infinity, leading to an unbounded and numerically unstable objective. Intuitively,  to improve the objective in \eqref{BC-good-bad}, it is necessary for the policy to allocate higher probabilities to trajectories in the good set while assigning lower probabilities to trajectories in the bad set. 

\noindent\textbf{Distribution Matching GB:} In the realm of distribution matching, the learning process entails a delicate balance. It involves minimizing the Kullback-Leibler (KL) divergence between the occupancy measures of the policy under consideration, denoted as $\rho^\pi$, and the good trajectories, represented as $\rho^{G}$. Simultaneously, the goal is to maximize the KL divergence between $\rho^\pi$ and the occupancy measure corresponding to bad trajectories, denoted as $\rho^{B}$. This dual divergence approach aims to shape the policy by aligning it closely with the good trajectories while also distancing it from the bad ones.
{These ``good'' and ``bad'' occupancy measures can be computed as:
$\rho^{G}(s,a) = (1-\gamma)\sum_{t=0}^\infty \gamma^t p_t(s,a|\Omega^G),
$
where $p_t(s,a|\Omega^G)$ is the probability that $(s_t,a_t)$ is in the set $\Omega^G$ and  $(s_t,a_t) = (s,a)$. Similarly, $\rho^{B}(s,a)$ can be computed in the same way.
 }Then, the training objective  becomes
{\small \begin{equation}
\label{DM-GB}\tag{\sf\small DM-GB}
\min_{\pi}\left\{ \lambda \KL\left(\rho^{\pi} ||\rho^{G} \right) -  (1-\lambda) \KL\left(\rho^{\pi} ||\rho^{B}   \right) \right\} 
\end{equation}}
Intuitively, to minimize the objective function in \eqref{DM-GB}, it is necessary for the occupancy distribution to move towards $\rho^{G}$ and far away from $\rho^{B}$. Consequently, $\rho^{\pi}$ will  allocate a higher probability to a pair $(s,a)$ that appears more frequently in $\Omega^G$ than in $\Omega^{B}$, and vice-versa. 

\subsection{Theoretical Insights}
We investigate the theoretical properties of our concept of learning from good and bad demonstrations.  Our aim is to explore the question whether we can obtain improved policies by learning from good and bad demonstrations. 
Since BC-GB works directly with trajectories, we will employ it to present our theory on why intuitively using good and bad trajectories is useful. Distribution Matching GB, on the other hand, works with state-action pairs, thus is much more challenging to analyze theoretically. That is why we develop our algorithm based on Distribution Matching GB, and  show extensive empirical results with it in our experimental results to demonstrate that it is a more practical algorithm and it outperforms existing work. 

We first note that, in the context of maximum likelihood estimation,
   $\pi^E$ is optimal for \eqref{BC}. 
In other words, if we have sufficient samples from the expert policy, it is guaranteed that we can recover the expert policy by solving \eqref{BC}.

We now look at the BC with good and bad trajectories in \eqref{BC-good-bad}. The following lemma says that a policy that allocates zero probabilities to bad trajectories in $\Omega^B$ will be optimal for \eqref{BC-good-bad}.
\begin{lemma}\label{lm-1}
 For any $\lambda>0$, if there exists a policy $\pi^*$ such that $P_{\pi^*}(\tau) = 0$ for all $\tau \in \Omega^B$, and $
     P_{\pi^*}(\tau) = \frac{P_{\pi^0}(\tau)}{  \sum_{\tau'\in \Omega^G}P_{\pi^0}(\tau')};~\forall \tau \in \Omega^G 
     $
     then $\pi^*$ is an optimal policy to \eqref{BC-good-bad}.
\end{lemma}
Where $P_{\pi^*}(\tau)$ is the probability of $\tau$ given by $\pi^*$, i.e., $P_{\pi^*}(\tau) = \sum_{(s_t,a_t,s_{t+1})\in\tau} \pi^*(a_t|s_t)P(s_{t+1}|a_t,s_t)$.
There might be no policy that allocates exactly $P_{\pi}(\tau) = 0$ for all $\tau\in \Omega^B$, due to, for instance, the dynamic of the environment or the structure of $\Omega^B$. However, intuitively, a policy trying to assign small probabilities to  $(s,a)$ that appear more frequently in $\Omega^B$ than in $\Omega^G$
will move towards  $\pi^*$ (so closer to the optimal policy). 

In the proposition below we show that, if we construct a bad set consisting of trajectories having low reward values and violating the cost constraints, then it is guaranteed that the optimal policy mentioned in Lemma \ref{lm-1} will perform better than the initial policy $\pi^0$ in terms of both reward and cost constraint satisfaction.

\begin{proposition}\label{th-1}
    For any $\lambda>0$, let $\pi^0$ be a pre-trained  feasible policy, $\bR^E = \bbE_{\tau\sim\pi^0}\Big[R(\tau)\Big]$, and  $\Omega^B$ be a collection  of trajectories of low reward and high-cost values $$\Omega^B = \left\{\tau\Big|~ R(\tau) \leq \bR^E ,~ C(\tau) >c_{\max}\}\Big]\right\}$$
the  optimal policy mentioned in Lemma \ref{lm-1} is feasible to the cost constraint while offering a better expected reward than the pre-trained policy $\pi^0$, specifically, 
{\small\begin{align}
    \bbE_{\pi^*} \left[R(\tau)\right] - \bbE_{\pi^0} \left[R(\tau)\right] &=\frac{\sum_{\tau}{P_{\pi^0}(\tau) (\bR^E-R(\tau))}}{1-P_{\pi^0}(\Omega^B)} \geq 0 \label{eq:p1-eq1} \\
    \bbE_{\tau \sim \pi^*} \left[C(\tau)\right] &\leq c_{\max}\nonumber
\end{align}}
where $P_{\pi^0}(\Omega^B) = \sum_{\tau\in \Omega^B}P_{\pi^0}(\tau)$.
\end{proposition}
The inequality  in  \eqref{eq:p1-eq1} suggests that increasing the proportion or total probability of the bad set $\Omega$ will result in a larger gap, thereby leading to improved policy enhancement. In other words, as more bad policies are identified, the quality of $\pi^*$ improves.

The above results hold for $\lambda > 0$, also indicating that one might obtain a better policy by eliminating the probabilities of bad trajectories (trajectories with low reward and high-cost values). When $\lambda = 0$, the BC is about to learn only from the bad trajectories. Interestingly, we can show that by just learning from bad trajectories, it is not necessary to obtain a better policy.  We first state the following lemma. 

\begin{lemma}\label{lm-2}
If $\lambda = 0$, any policy $\pi^*$ such that $P_{\pi^*}(\tau) = 0$ for all $\tau \in\Omega^B$ is optimal for \eqref{BC-good-bad}.  
\end{lemma}

The following proposition tells us that learning with $\lambda = 0$ would not offer a policy improvement as in the case of $\lambda>0$.
\begin{proposition}\label{th-2}
If $\lambda = 0$, and the bad set $\Omega^B$ is selected in the same manner as in Proposition \ref{th-1}, then the optimal policy $\pi^*$ from Lemma \eqref{lm-2} is feasible, but it does not necessarily provide a higher expected reward than the policy $\pi^0$.
\end{proposition}

In Propositions \ref{th-1} and \ref{th-2}, it is assumed that the pre-trained policy $\pi^0$ is feasible. It is then relevant to discuss other scenarios where the expected policy is not feasible, or even when the cost function is unknown. We summarize our claims below.
\begin{proposition}\label{th-3}
    The following hold
    \begin{itemize}
        \item[(i)] If we select the bad set as $\Omega^B = \left\{\tau\Big|~ R(\tau) \leq \bR^E ,~ C(\tau) >\bC^E\}\Big]\right\}$, then it is guaranteed that  $\pi^*$ offers a higher (or equal) expected reward and lower (or equal) expected cost, compared to those from $\pi^0$, where  $\bC^E = \bbE_{\tau\sim \pi^0} [C(\tau)]$ 
        \item[(ii)] If the pre-trained policy $\pi^0$ is not feasible, then if we select the bad set as $\Omega^B = \left\{\tau\Big| C(\tau) >c_{\max}\}\Big]\right\}$, then it is guaranteed that  $\pi^*$ is feasible 
        \item[(iii)] If the cost function is not accessible, but there is an oracle that can tell us which trajectories are violating the constraint,  then by selecting, $\Omega^B = \left\{\tau\Big| \tau \text{ is violated }\Big]\right\}$, then $\pi^*$ is feasible.
    \end{itemize}
\end{proposition}

The above results provide some interesting insights to understand the framework.  It is evidenced that if we learn from both good and bad trajectories, the policy will be trained towards a better one, compared to the pre-trained policy $\pi^0$. If we only use bad trajectories, then it is possible that we cannot get a better policy than $\pi^*$. This remark will be further validated in our later experiments, as we observe that one cannot learn a good policy by just using bad trajectories.
Moreover, according to Proposition \ref{th-3}, our framework can be used to train towards policies with lower cost or even feasible policies by selecting different bad sets $\Omega^E$, even when the cost function is not known beforehand.

The theory also tells us that if we are more selective in choosing good and bad trajectories, we will tend to obtain better policies, as long as there are policies that can eliminate the probabilities of the bad ones. However, the selection process can be tricky and may not be easy to achieve in practice. So, it is better to not be too selective (or conservative) in classifying good and bad demonstrations.

\subsection{Example}
\begin{figure}[htb] 
\centering
    \includegraphics[width=0.6\linewidth]{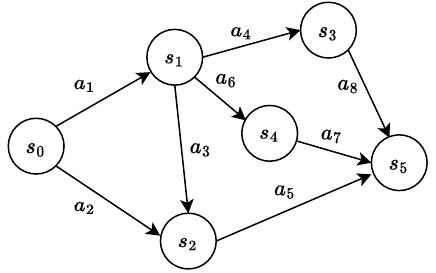} 
    \caption{Example} 
    \label{fig:SMDP} 
\end{figure}
We give a small example to demonstrate how our framework returns a better policy by learning from bad and good trajectories. We consider the small deterministic MDP given in Figure \ref{fig:SMDP}. 
The rewards, $r$, costs, $c$ and pre-trained policy $\pi^0$ are as shown in Table~\ref{rc}. The probabilities are over feasible actions from the state. There are 4 possible trajectories  $\tau_1 = \{s_0,s_1,s_3,s_5\},~\tau_2 = \{s_0,s_1,s_4,s_5\},~\tau_3 = \{s_0,s_1,s_2,s_5\},$ and $\tau_4 = \{s_0,s_2,s_5\}$.
\begin{table}[]
    \centering
    \begin{tabular}{|c|c|c|c|c|c|c|}
    \hline
         & $s_0$ & $s_1$ & $s_2$ & $s_3$ & $s_4$ & $s_5$ \\
         \hline
      $r$& 0 & 2 & 3 & 1 & 2 & 0 \\
      \hline
      $c$ & 0 & 1 & 1 & 3 & 1 & 0 \\
      \hline
      $\pi^0$ & 1/2, 1/2 & 1/3, 1/3, 1/3 & 1 & 1 & 1 & 1 \\
      \hline
    \end{tabular}
    \caption{Rewards, Costs and Policy}
    \label{rc}
\end{table}
We then see that 
$\bbE_{\pi^0}[R(\tau)] = \sum_{\tau}P_{\pi^0}(\tau) = 3.5;~~
    \bbE_{\pi^0}[C(\tau)] = \sum_{\tau}P_{\pi^0}(\tau) = 2,$
 implying that $\pi^0$ is feasible for the CMDP problem. 


Under our good-bad scheme, trajectory $\tau_1 = \{s_0,s_1,s_3,s_5\}$ has the accumulated reward and cost as  $R(\tau_1) = 3 < \bbE_{\pi^0}[R(\tau)], C(\tau_1) = 4>c_{\max}$. So, according to the criteria in Theorem \ref{th-1}, $\tau_1$ should be considered a bad trajectory (the others are good). The BC objective can be written as
$F(\pi) = \lambda  \sum_{i\in\{2,3,4\}} P_{\pi^0}(\tau_i) \ln P_{\pi}(\tau_i) $. 
The following policy $\pi^*$ such that $\pi^*(a_4|s_1) = 0$, $\pi^*(a_6|s_1) = \pi^*(a_3|s_1) = 1/2$, $\pi^*(a_1|s_0) = 2/5$ and $\pi^*(a_2|s_0)=3/5$ will satisfy  the condition in Lemma \ref{lm-1}, i.e.., 
    $P_{\pi^*}(\tau_1);~ P_{\pi^*}(\tau_2) = 1/5 = \frac{P_{\pi^0}(\tau_2)}{5/6};~
    P_{\pi^*}(\tau_3) = 1/5 = \frac{P_{\pi^0}(\tau_3)}{5/6};~ P_{\pi^*}(\tau_4) = 3/5 = \frac{P_{\pi^0}(\tau_4)}{5/6}~
    \text{and } P_{\pi^0}(\tau_2)+ P_{\pi^0}(\tau_3) + P_{\pi^0}(\tau_4) = 5/6,$
thus $\pi^*$ is optimal for $\max_{\pi}\{F(\pi)\}$. On the other hand,
$
\bbE_{\pi^*}[R(\tau)] = 3.6; ~ \bbE_{\pi^*}[C(\tau)] = 1.6.
$
So $\pi^*$ offers a better expected reward and a lower cost compared to the pre-trained policy $\pi^0$.

\section{Self-Imitation based Safe RL}\label{sec:algorithm}
In this section, we present a practical IL-based algorithm for constrained RL.
A BC-based algorithm can be developed using  \eqref{BC-good-bad}. However, this approach (or the BC in general) would  not be practical and would necessitate a huge number of samples to attain the desired performance. In contrast, Distribution Matching proves to be a more practical alternative. Taking inspiration from the GAIL algorithm, to address \eqref{DM-GB}, one can construct two discriminators: one for $\KL\left(\rho^{\pi} ||\rho^{G} \right)$ and another for $\KL\left(\rho^{\pi} ||\rho^{B} \right)$. Nonetheless, this approach involves  two adversaries and would be highly unstable (as demonstrated in our experiments). To get rid of adversarial training, let us
put the occupancy measures of the learning policy and the good demonstrations together, and
consider the following mixed state-action distribution 
$\rho^{G,\pi} = ({\rho^{\pi} + \rho^{G} })/{2}$. We then set our aim to maximize the KL divergence between $\rho^{G,\pi}$ and the occupancy measure of the bad trajectories $\rho^{B}$ (thus making  $\rho^{G,\pi}$ far away from the ``bad'' occupancy measure $\rho^{B}$). 
\begin{equation}\label{eq:KL}
\max_{\pi} \left\{\KL\left(\rho^{G,\pi}||\rho^{B}\right)\right\} = \max_{\pi}\bbE_{(s,a)\sim \rho^{G,\pi}}\left[\ln\frac{\rho^{B}(s,a)}{\rho^{G,\pi}(s,a)}\right]    
\end{equation}
To estimate distribution ratio $\frac{\rho^{B}(s,a)}{\rho^{G,\pi}(s,a)}$, we propose the following surrogate maximization problem 
{\small
\label{eq:learn_classifer}
\begin{align}
    &\max_{K:S\times A\rightarrow (0,1)} \Big\{J(K,\pi):=\bbE_{\rho^B}[\ln(K(s,a))]\nonumber\\
    &+\frac{1}{2}\bbE_{\rho^{\pi}}[\ln(1-K(s,a))]+\frac{1}{2}\bbE_{\rho^{G}}[\ln(1-K(s,a))]  \Big\}\label{eq:DM-GD-eq1}
\end{align}}
Here, \eqref{eq:DM-GD-eq1} is connected to \eqref{eq:KL} through the following result:
\begin{proposition}\label{prop:DM}
    The maximization in \eqref{eq:DM-GD-eq1} is achieved at $K^*(s,a)$ such that  
    \[
    \ln \left(\frac{K^*(a,s)}{1-K^*(s,a)}\right) = \ln\frac{\rho^{B}(s,a)}{\rho^{G,\pi}(s,a)}
    \]
\end{proposition}
This implies that the  distribution ratio  can be estimated as  $\ln\frac{K^*(s,a)}{1-K^*(s,a)}$. As a result, the policy can be updated by maximizing $\max_{\pi}\bbE_{(s,a)\sim \rho^{G,\pi}}\left[\ln\frac{K^*(s,a)}{1-K^*(s,a)}\right]$, which is equivalent to $\max_{\pi}\bbE_{(s,a)\sim \rho^{\pi}}\left[\ln\frac{K^*(s,a)}{1-K^*(s,a)}\right]$  as the occupancy measure of the good demonstrations is constant.  

In practice, $K(s,a)$ need not be fully optimized. Instead, $K$ and $\pi$ can be updated alternatively by gradient ascent. It is important to note that we update $K(s,a)$ by maximizing $J(K,\pi)$  and update $\pi$ by maximizing  $\KL\left(\rho^{G,\pi}||\rho^{B}\right)$, so our algorithm  is \textit{non-adversarial}. In other words, $K(s,a)$ operates in a cooperative manner rather than an adversarial one  -- it collaborates with the policy $\pi$ to estimate the distribution ratio and  make  the mixed distribution $\rho^{G,\pi}$ far away from the bad one $\rho^B$. Here, the non-adversarial nature of our method stems from our approach of maximizing the KL divergence, in contrast to the minimizing aspect employed in GAIL.
Drawing from the above analyses, we proceed to outline our algorithm. Let $w$ and $\theta$ denote the parameters of $K(s,a)$ and $\pi(s,a)$ respectively. The core concept involves iteratively enhancing $K$ and $\pi$ through alternating gradient ascent updates. 
While $K_w$ can be updated by using the derivatives of $J(K_w,\pi_\theta)$, $\pi_\theta$ can be updated by a policy gradient method, e.g., PPO \cite{schulman2017proximal}.
During the training process, we generate additional trajectories and update  the good and bad sets. The  key steps of our method are described in Algorithm \ref{alg:good_bad_algo}.


\begin{algorithm}[htbp]
\caption{Self-imitation Safe Reinforcement Learning }\label{alg:good_bad_algo}

\begin{algorithmic}
\REQUIRE $\pi^0$,$K_{\omega}$,$R_{G}$,$R_{B}$ ,$c_{max}$, learning rates $\kappa_\theta,\kappa_w$
\STATE ${\Omega_{G}} \gets \emptyset$;~${\Omega_{B}} \gets \emptyset$;~ $\pi_\theta \gets \pi^0$
\WHILE{not converge}
    \STATE \# Sample new set trajectories
    \STATE $T = \{\tau_{0},\tau_{1},...,\tau_{n}\sim \pi_{\theta}\}$
    \STATE \# Update the ``good'' and ``bad'' sets
    \STATE $R_B \gets \bbE_{\tau \sim T}\left[R(\tau)\right]-\sigma_{\tau \sim T}[R(\tau)]${, \#$\sigma$ is the deviation}
    \STATE ${\Omega_{G}} \gets {\Omega_{G}} \cup \left\{\tau \in T| R(\tau)\geq R_{G} \cap C(\tau)\leq c_{max}\right\}$ 
    \STATE ${\Omega_{B}} \gets {\Omega_{B}} \cup \left\{\tau\in T| R(\tau) < R_{B} \cup C(\tau) > c_{max}\right\} $
    \STATE \# Update $K_w$
    \STATE $w \gets w + \kappa_w \nabla_w J(K_w,\pi_\theta)$
    \STATE where $J(K_w,\pi_\theta) = \bbE_{\Omega^B}[\ln(K_w(s,a))]
    +\frac{1}{2}\bbE_{T}[\ln(1-K_w(s,a))]+\frac{1}{2}\bbE_{\Omega^{G}}[\ln(1-K_w(s,a))]$
    \STATE \# Update $\pi_\theta$
    \STATE $\theta = \theta + \kappa_\theta \bbE_{\tau\sim T}\left[\nabla_{\theta} \ln\pi_\theta(s,a)Q^K(s,a) \right]$
    \STATE where $Q^K(s,a) = \bbE_{T}\left[\sum_t \gamma^t ln \frac{K_w(s_t,a_t)}{1-K_w(s_t,a_t)}\Big|\substack{s_0=s\\a_0=a}\right]$
\ENDWHILE
\end{algorithmic}
\end{algorithm}

\section{EXPERIMENTS}
\begin{figure*}[htb]
\centering
\rotatebox[origin=c]{90}{\centering Return}
\showReturn[0.15]{0}{25}{25}{25}{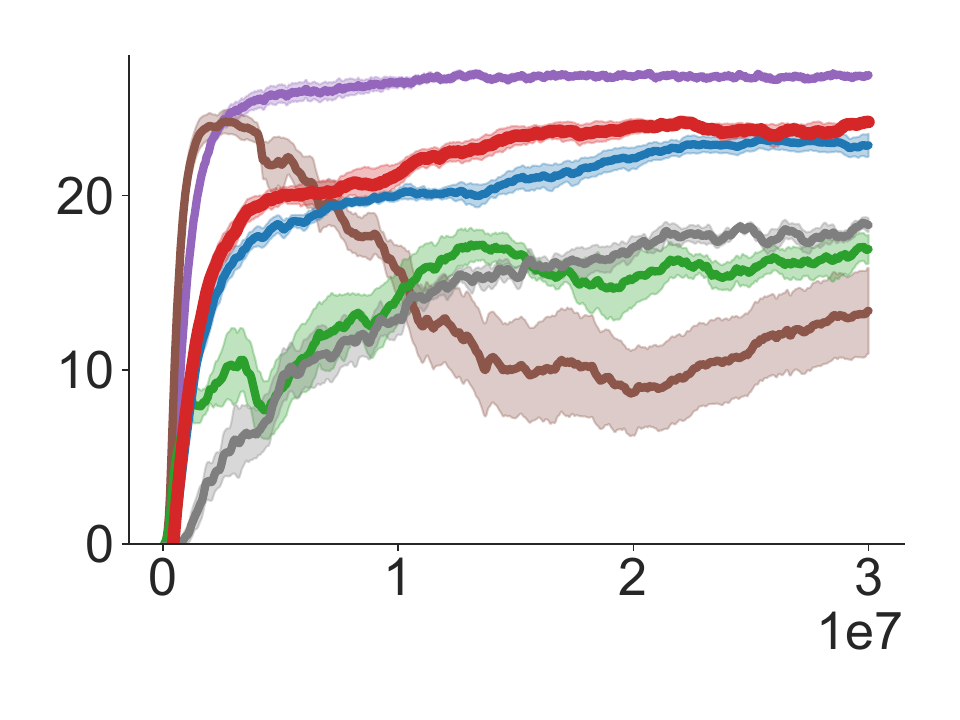}{SafetyPointGoal}
\showReturn[0.15]{0}{25}{25}{25}{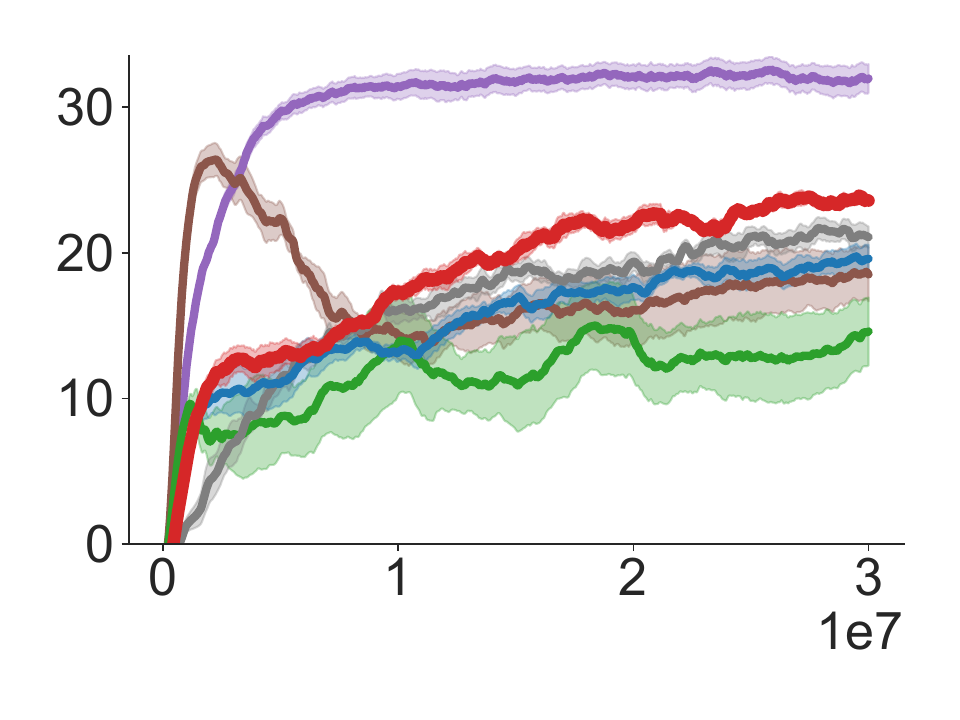}{SafetyCarGoal}
\showReturn[0.15]{0}{25}{25}{25}{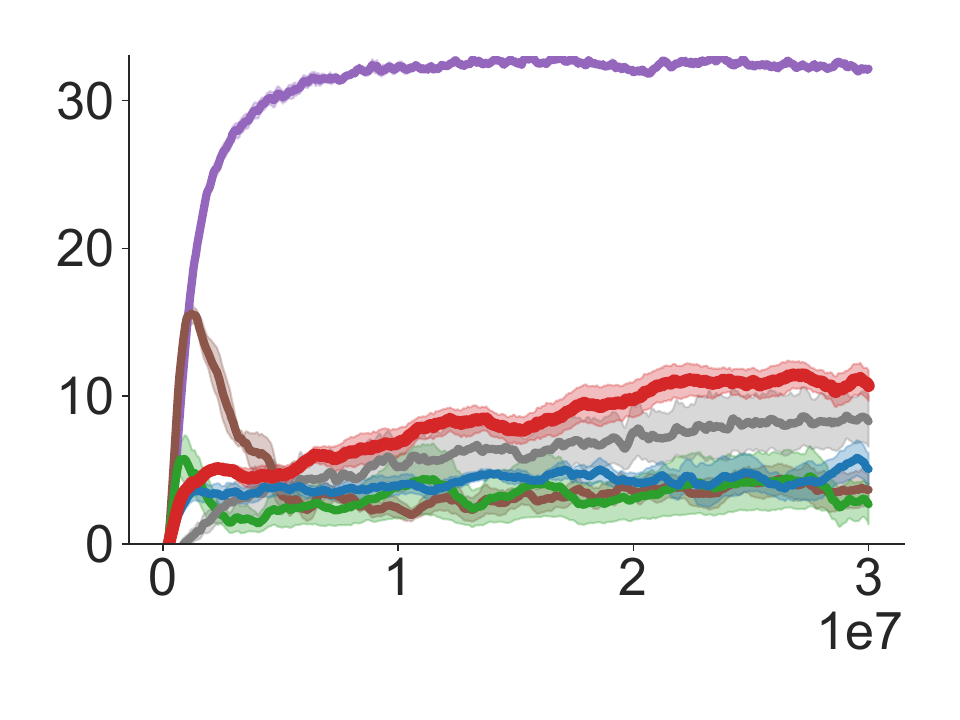}{SafetyPointButton}
\showReturn[0.15]{0}{25}{25}{25}{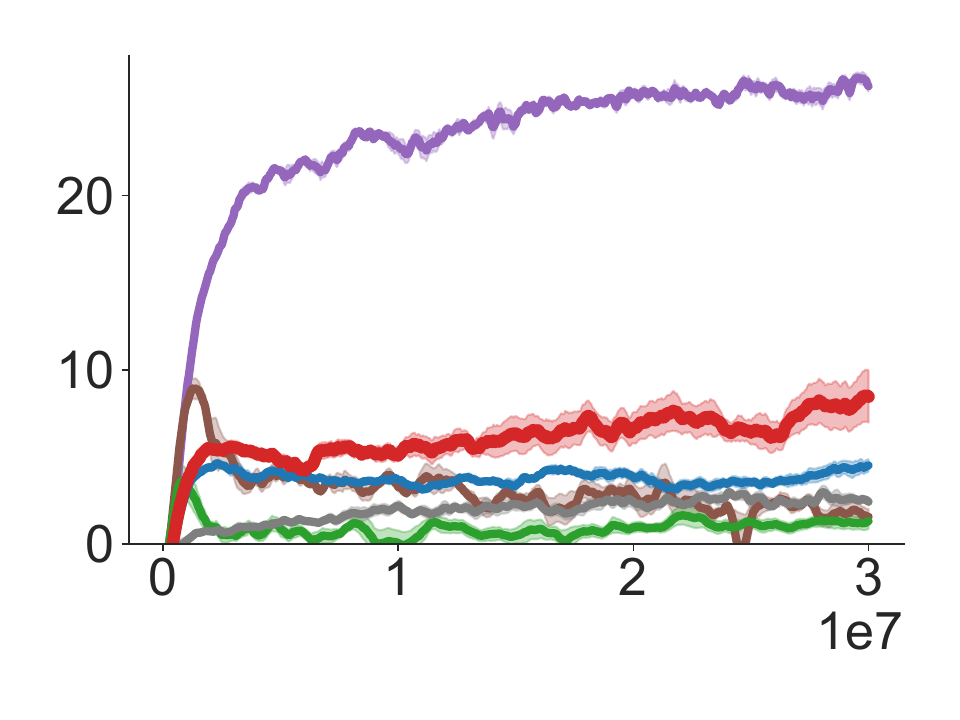}{SafetyCarButton}
\showReturn[0.15]{25}{25}{25}{25}{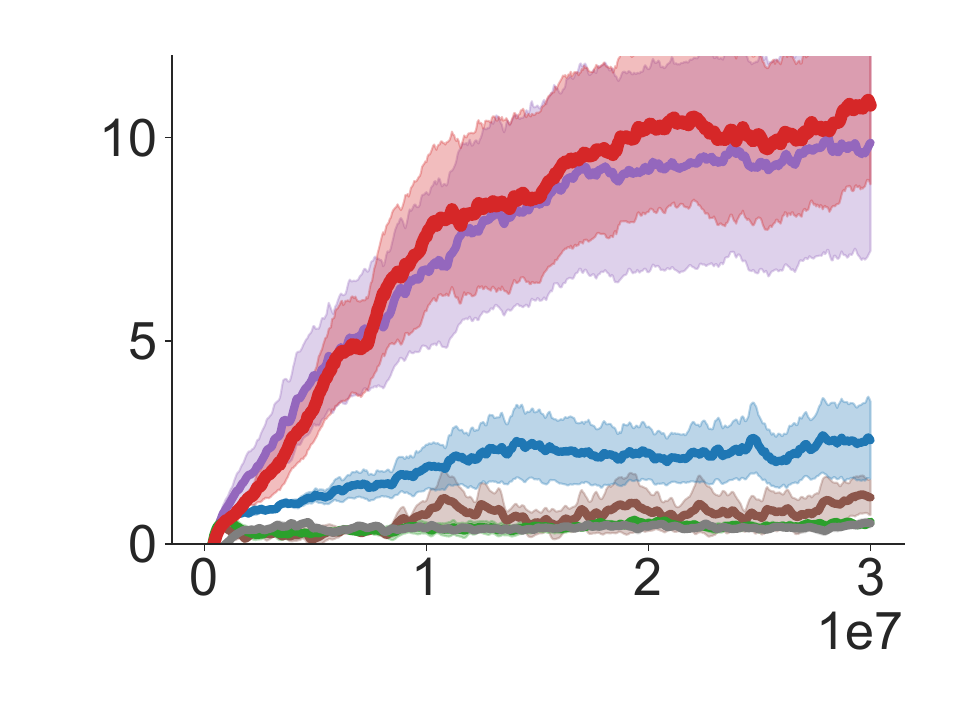}{SafetyPointPush}
\showReturn[0.15]{0}{25}{25}{25}{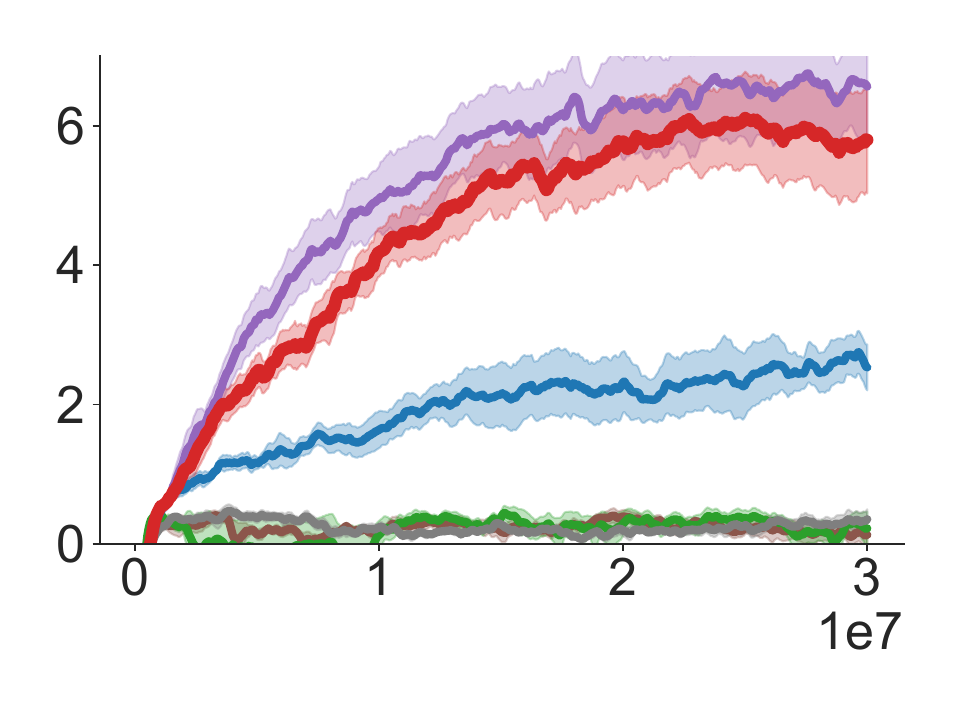}{SafetyCarPush}

\rotatebox[origin=c]{90}{\centering Cost}
\showCost[0.15]{0}{25}{25}{25}{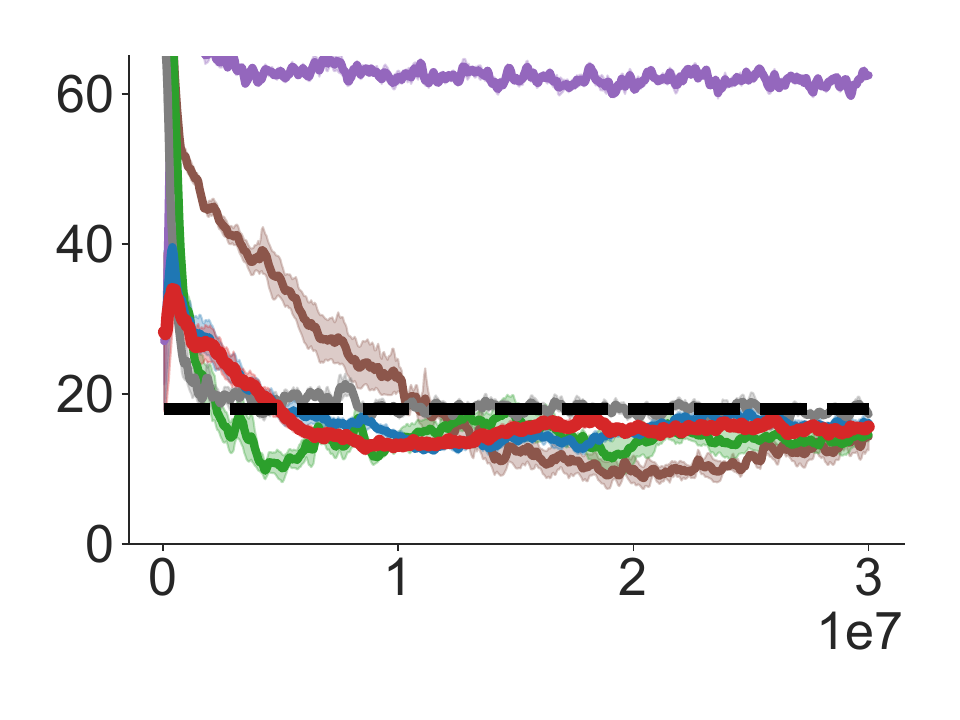}
\showCost[0.15]{0}{25}{25}{25}{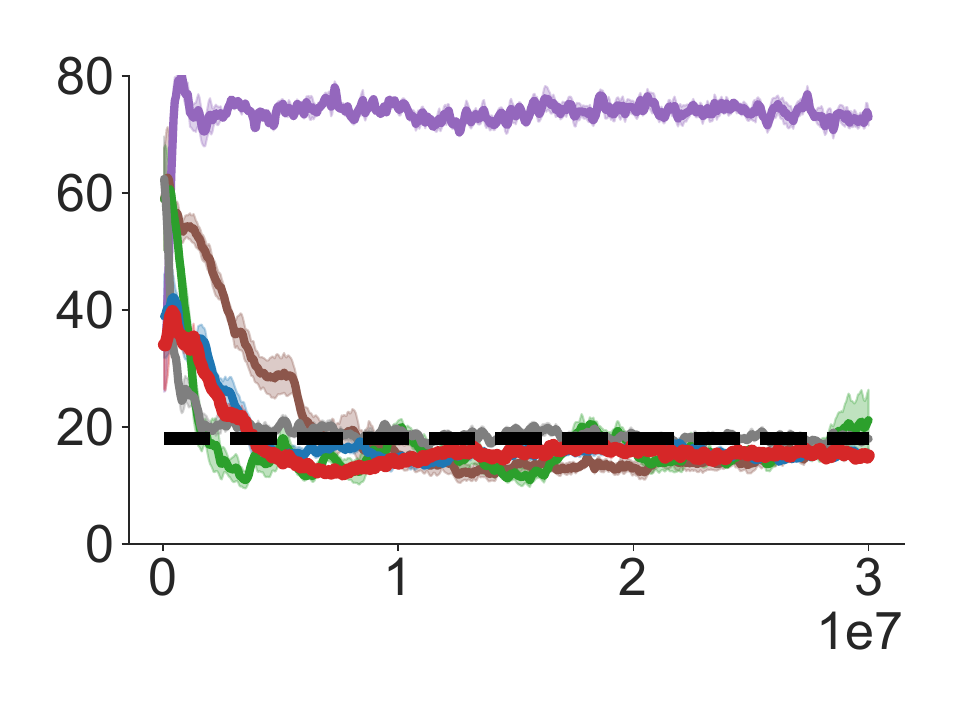}
\showCost[0.15]{25}{25}{25}{25}{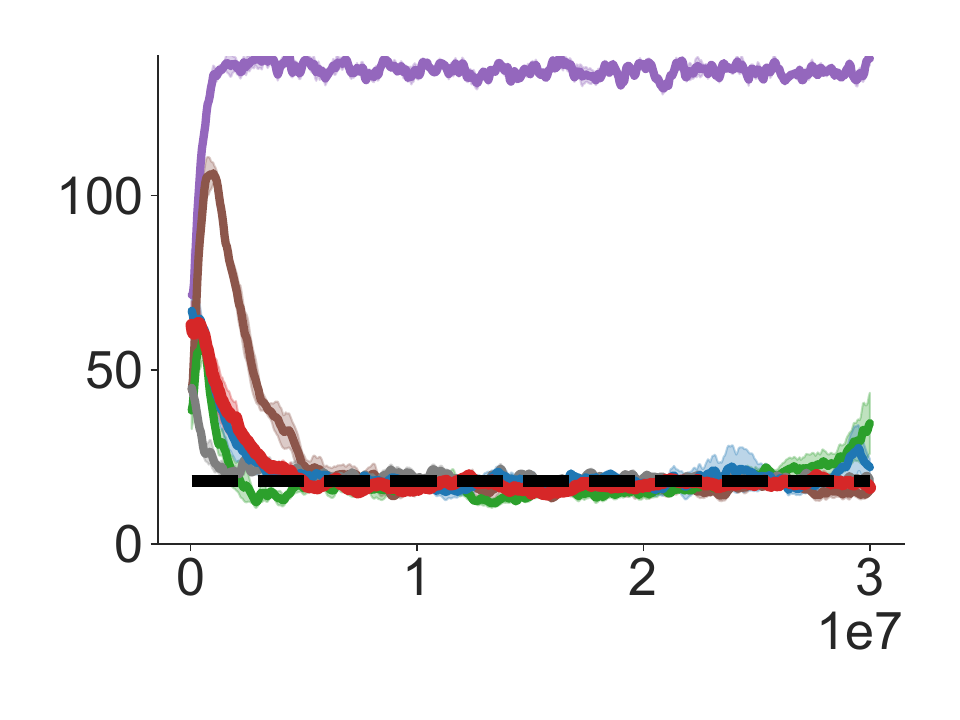}
\showCost[0.15]{25}{25}{25}{25}{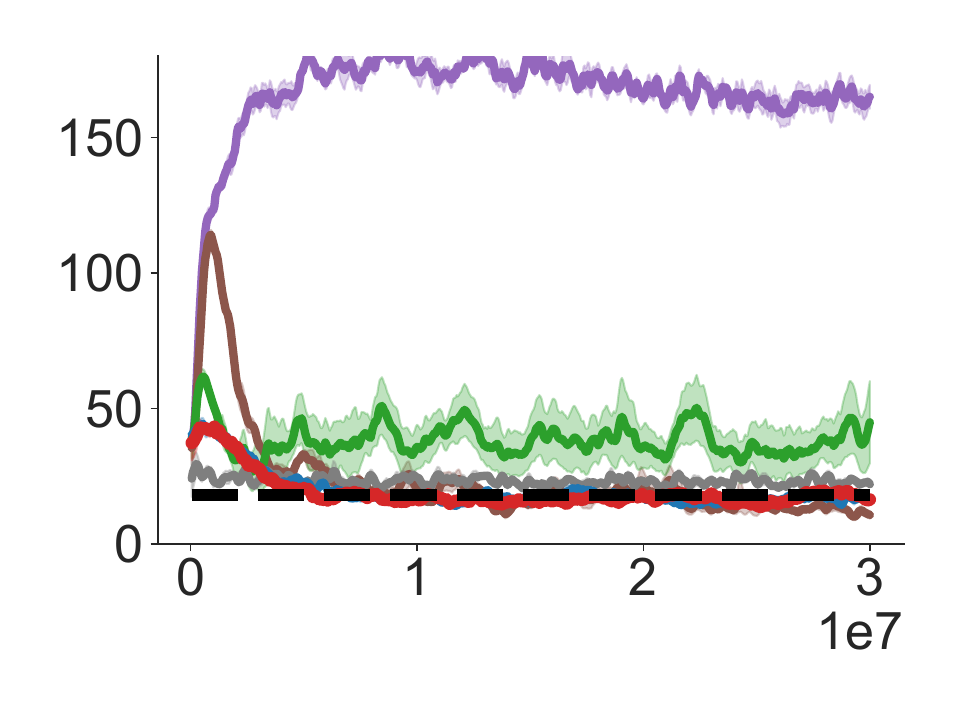}
\showCost[0.15]{10}{25}{25}{25}{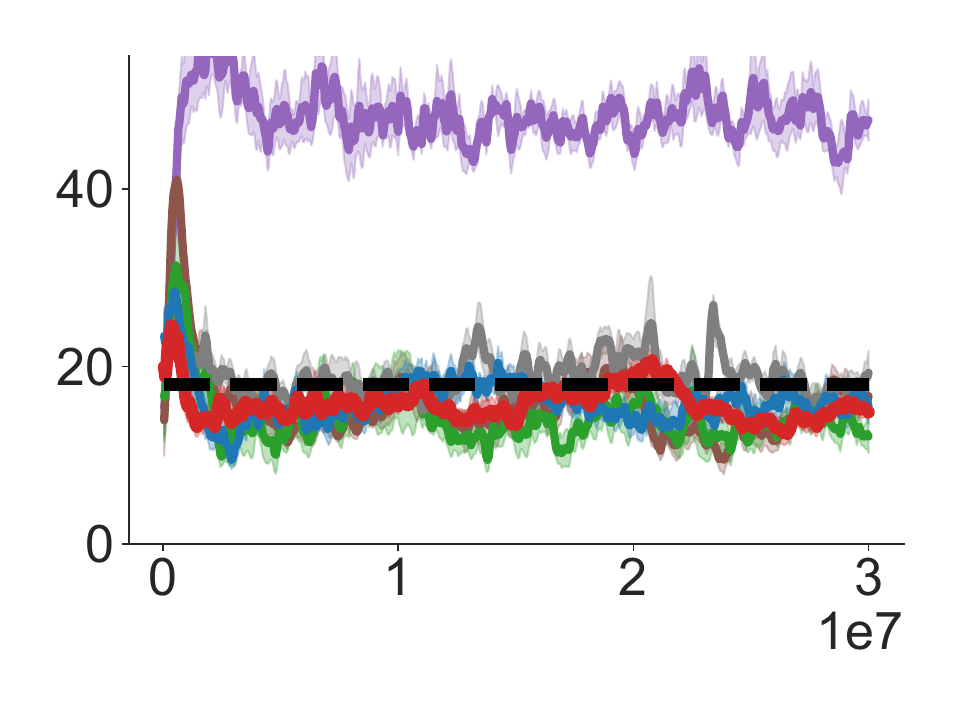}
\showCost[0.15]{25}{25}{25}{25}{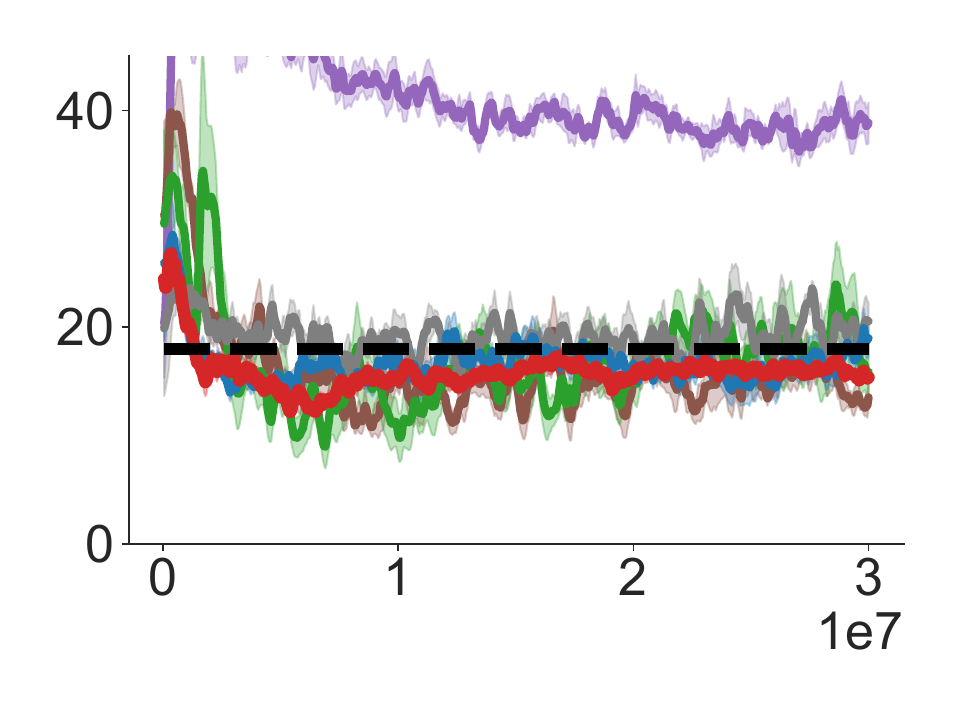}
\showLegend[1.0]{10}{10}{20}{20}{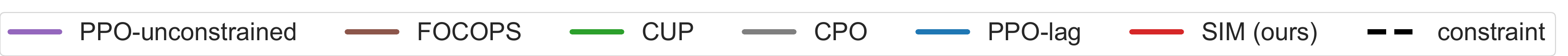}
\caption{Training curves for 6 different SafetyGym environments. Every lines in calculated by the mean with shaded by the standard error of 6 independent seeds.}
\label{fig:Safety_gym_results}
\end{figure*}

We conduct experiments to compare our method against some state-of-the-art Constrained RL algorithms: FOCOPS~\cite{zhang2020first}, CUP~\cite{yang2022constrained}, CPO~\cite{achiam2017constrained}\footnote{FOCOPS, CUP, and CPO implementations can be found on \url{https://github.com/PKU-Alignment/omnisafe}.}. For the sake of completeness, we also include PPO-Lagrangian ~\cite{ray2019benchmarking} and unconstrained PPO. 
We use PPO-Lagrangian to train our pre-trained policy and name our algorithm as  SIM, standing for \textbf{S}elf-\textbf{IM}itation based safe RL algorithm. 
Through the following experiments, we aim to address the following questions:
(Q1) Would SIM outperform state-of-the-art constrained RL algorithms? (Q2) Is it necessary to use both good and bad demonstrations in the training? (Q3) How is SIM compared to a BC-based  and GAIL-based algorithm? (Q4) How does SIM perform with different expertise levels of the initial policy $\pi^0$? Can it benefit from a not well-trained policy?

We set the  cost limit as $c_{max} = 18$. We, however, train the initial policy $\pi^0$ with a higher cost limit of $c_{max}=28$, which allows us to generate more trajectories
of high rewards and more unsafe trajectories. 
We test our method on 6 SafetyGym environments~\cite {Safety-Gymnasium}. We also simplify the names of environments, e.g., SafetyPointGoal1-v0 is renamed as SafetyPointGoal. 

\subsection{SIM vs other Constrained RL methods on SafetyGym}
We compare our algorithm with prior safe RL ones using six different SafetyGym environments~\cite{ray2019benchmarking,Safety-Gymnasium}. The learning curves are shown in Figure~\ref{fig:Safety_gym_results} where the experiments are repeated over 6 independent seeds. For the sake of comparison, we include the PPO-unconstrained. Here are the key observations. In all the 6 environments, including the very challenging ones (SafetyPointButton and SafetyCarButton), SIM achieves the best performance -- it offers the highest expected rewards while being safe. The PPO-unconstrained gives the highest rewards but is unsafe by a huge margin. Notably, in the last two push tasks, SIM achieves competitive or even higher rewards, compared to the unconstrained one. Overall PPO-Lagrangian had the second best performance. 

\subsection{SIM vs GAIL, and the Importance of ``Good'' and ``Bad'' Demonstrations}
\begin{figure}[htb]
\centering
\rotatebox[origin=c]{90}{\centering Return}
\showReturn[0.21]{0}{0}{0}{0}{Images/SafetyCarButton1-v0/onlyGB_return}{SafetyCarButton}
\showReturn[0.21]{0}{0}{0}{0}{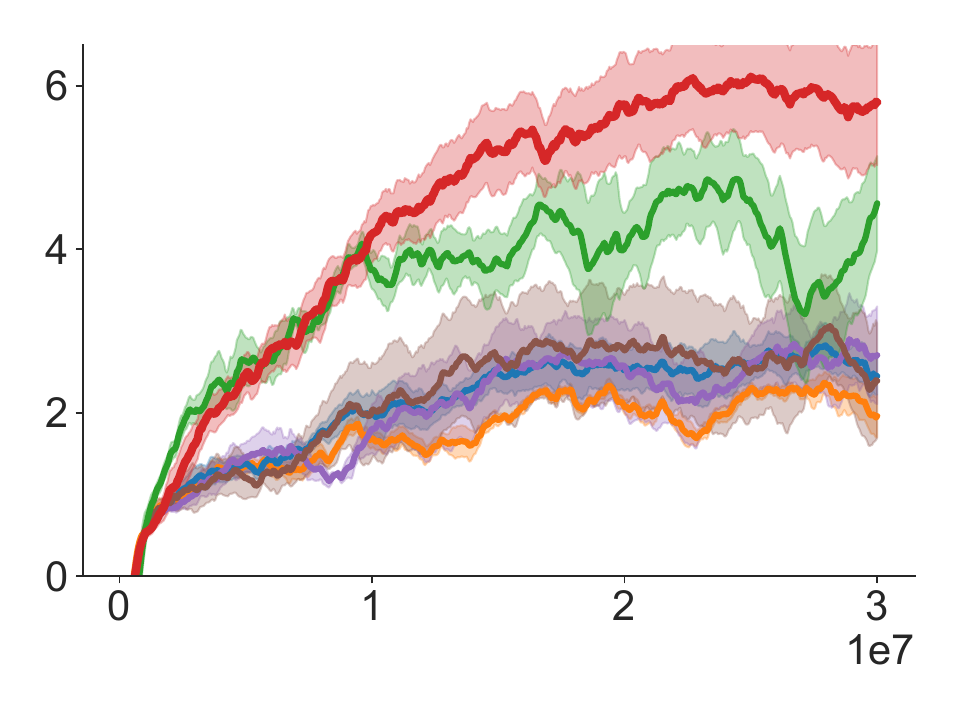}{SafetyCarPush}

\rotatebox[origin=c]{90}{\centering Cost}
\showCost[0.21]{0}{0}{0}{0}{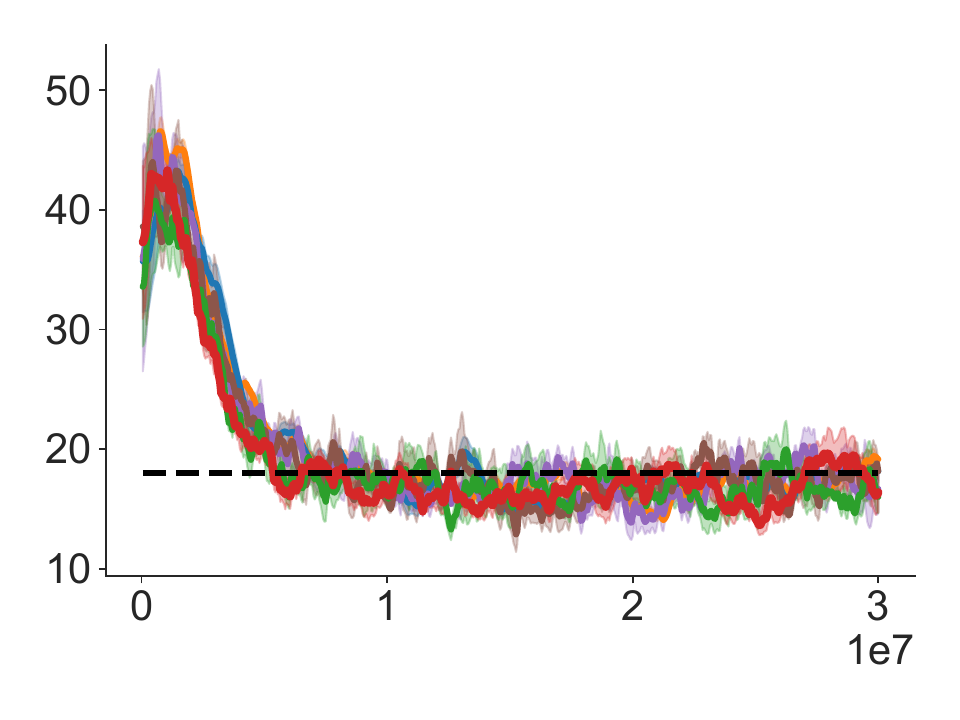}
\showCost[0.21]{0}{0}{0}{0}{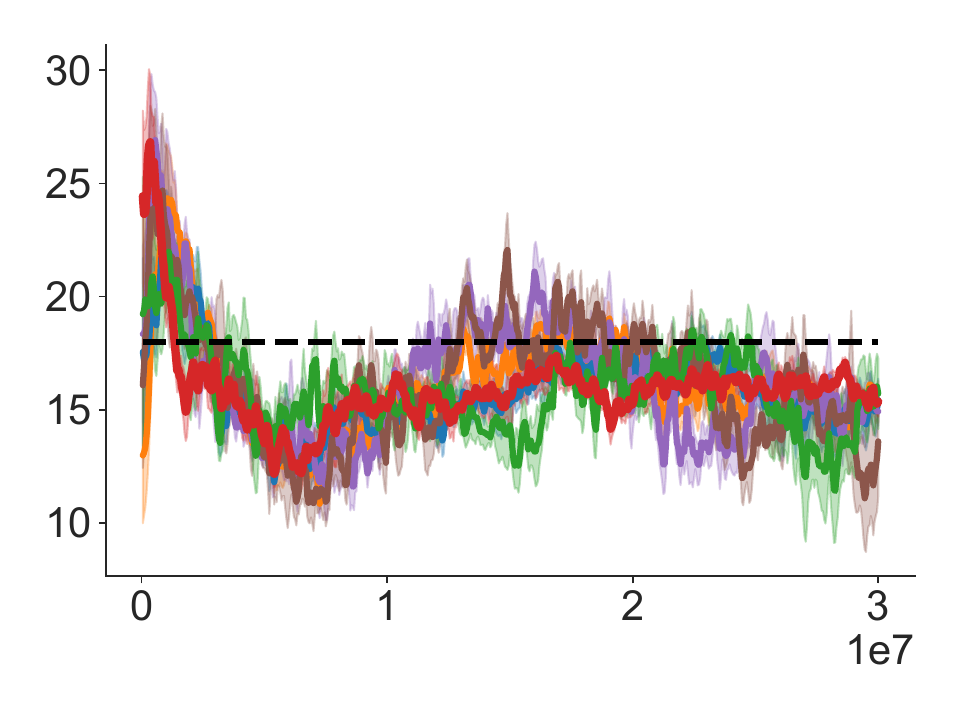}
\showLegend[0.45]{10}{10}{20}{20}{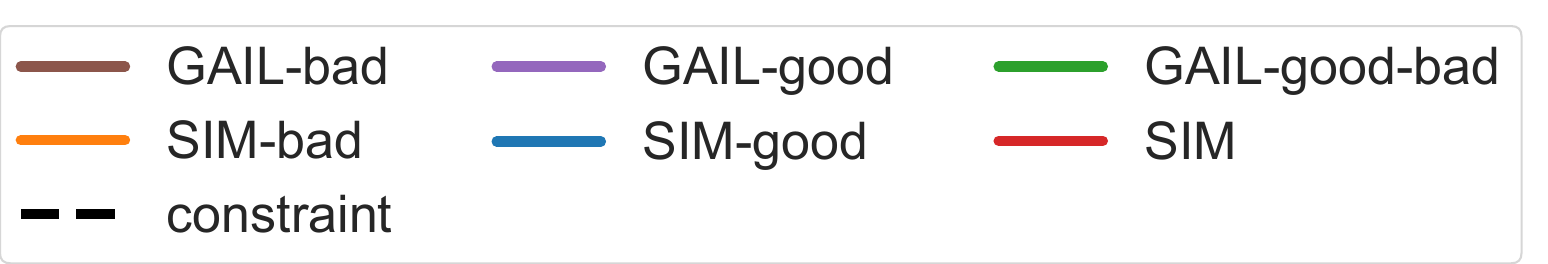}
    \caption{Comparisons with GAIL-based algorithms and other versions of SIM. 
    }
    \label{fig:GB_results}  
\end{figure}

We aim to assess the importance of having both ``good'' and ``bad'' demonstrations in our IL-based approach, as well as to demonstrate the advantages of our non-adversarial method. To 
this end, we compared SIM with three versions of GAIL that use (i) only good demonstrations $\Omega^G$, (ii) only ``bad'' demonstrations $\Omega^B$, and (iii) both ``good'' and ``bad'' demonstrations, but using two discriminators as described in Section \ref{sec:algorithm}. For the sake of comparison, we also include two versions of SIM with only good demonstrations and only bad demonstrations. We do this by just removing the ``good'' (or ``bad'') part from \eqref{eq:DM-GD-eq1}. 
Figure~\ref{fig:GB_results} shows the comparisons on 2 SafetyGym environments, which clearly demonstrates the superior performance of SIM, compared to 
 SIM versions with only good (or bad) demonstrations, and the 3 different GAIL versions, highlighting the importance of having both good and bad trajectories in the training.
 Furthermore, it can be observed that our non-adversarial algorithm is highly stable and consistent in returning high-reward and safe policies.


\subsection{SIM vs Behavioral Cloning}\label{sec:experiments - BCGD}
As mentioned earlier, one  can design an IL-based method for constrained RL based on \eqref{BC-good-bad}. In this section, we aim to compare our algorithm with a BC-based approach. We implement two BC algorithms: one is based on \eqref{BC} with one ``good'' demonstration, and another is based on \eqref{BC-good-bad} with both sets.
The comparison results are presented in Figure~\ref{fig:BC_results}. For the two BC algorithms, since their training curves are not comparable with those from SIM,  we only draw horizon lines representing their expected reward and cost at convergence (their training curves are provided in the appendix).
\begin{figure}[htbp]
\centering
\rotatebox[origin=c]{90}{\centering Return}
\showReturn[0.21]{0}{0}{0}{0}{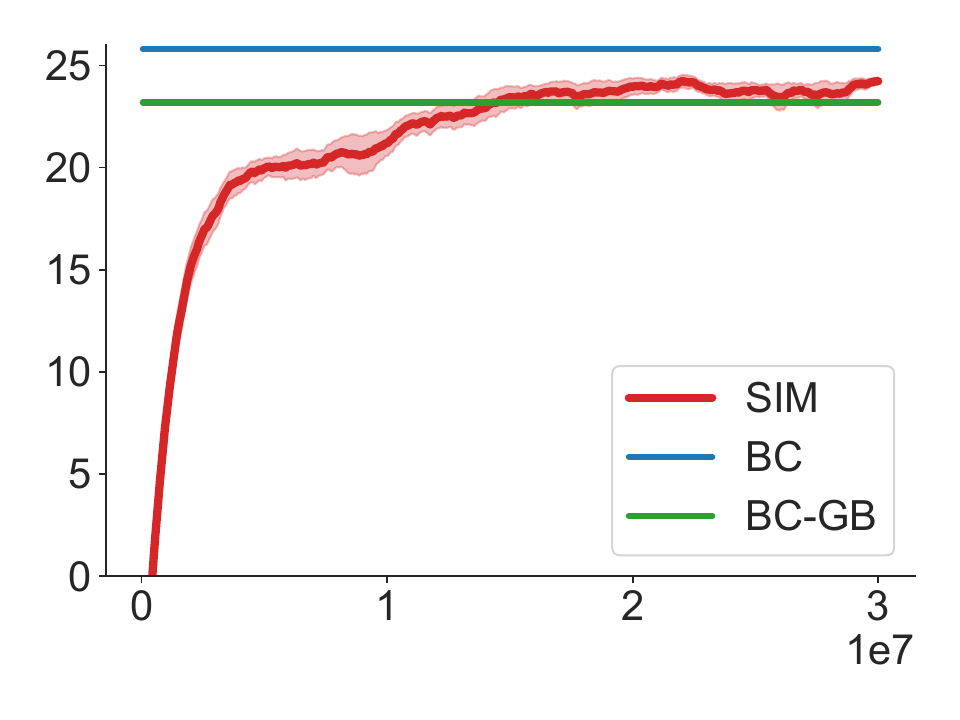}{SafetyPointGoal}
\showReturn[0.21]{0}{0}{0}{0}{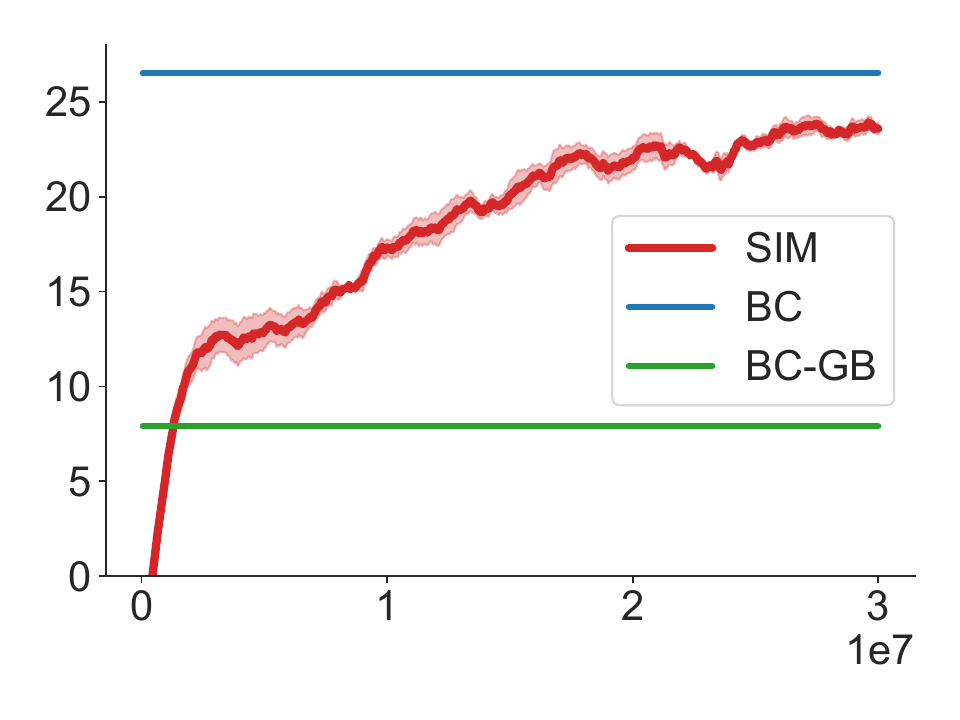}{SafetyCarGoal}

\rotatebox[origin=c]{90}{\centering Cost}
\showCost[0.21]{0}{0}{0}{0}{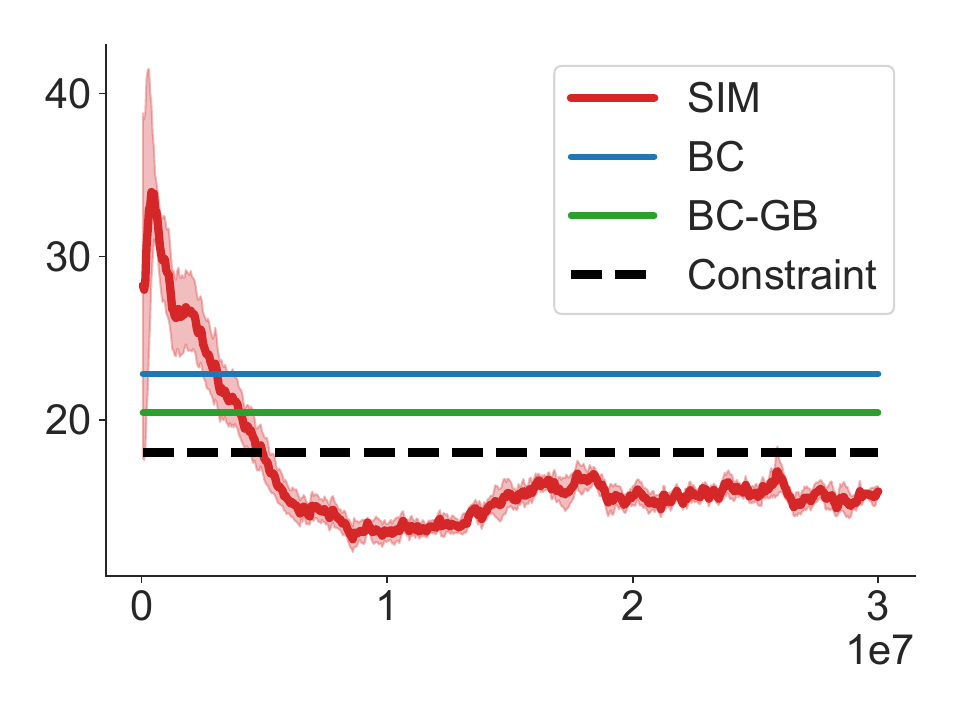}
\showCost[0.21]{0}{0}{0}{0}{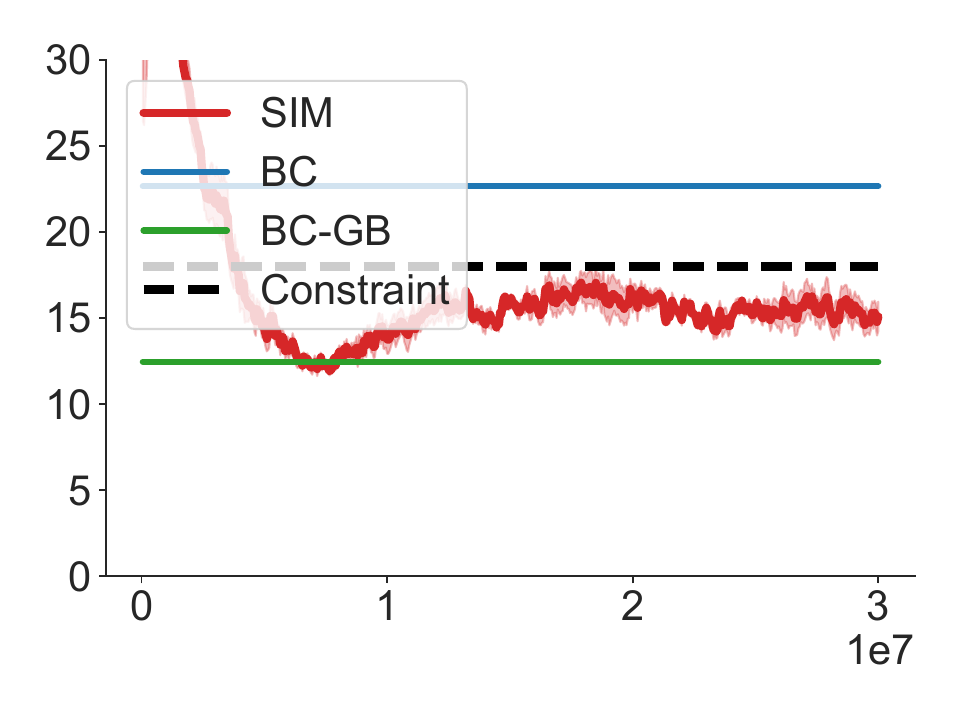}

    \caption{Comparison with BC-based algorithms.}
    \label{fig:BC_results}
\end{figure}

In both environments, the BC achieves the highest expected rewards, but it fails to satisfy the constraint. BC-GB either gives low-reward or unsafe policies. 
On the other hand, SIM consistently achieves high rewards while satisfying the constraint in all the experiments. 


\subsection{Varying Expertise Level}\label{sec:varying E}
In this section, our goal is to comprehend the influence of the training extent of the pre-trained policy $\pi^0$ on the efficiency of SIM. To this end, we trained the initial policy $\pi^0$ using varying numbers of environmental steps: specifically, 10 million (1e7), 20 million (2e7), and 30 million (3e7) steps, corresponding to what we term ``entry-level'', ``medium-level'' and ``expert-level'', respectively.
 The comparison results are shown in Figure~\ref{fig:diff_expert_results}, revealing that both the ``entry-level'' and 
``medium-level'' pre-trained policies achieve lower expected rewards  compared to the ``expert-level'' SIM. Nevertheless, with either the  ``entry-level'' or  
``medium-level'' $\pi^0$, SIM outperforms the original PPO-Lagrangian baseline, which was the second best among all the baselines. 

\begin{figure}[htbp]
\centering
\rotatebox[origin=c]{90}{\centering Return}
\showReturn[0.21]{0}{0}{0}{0}{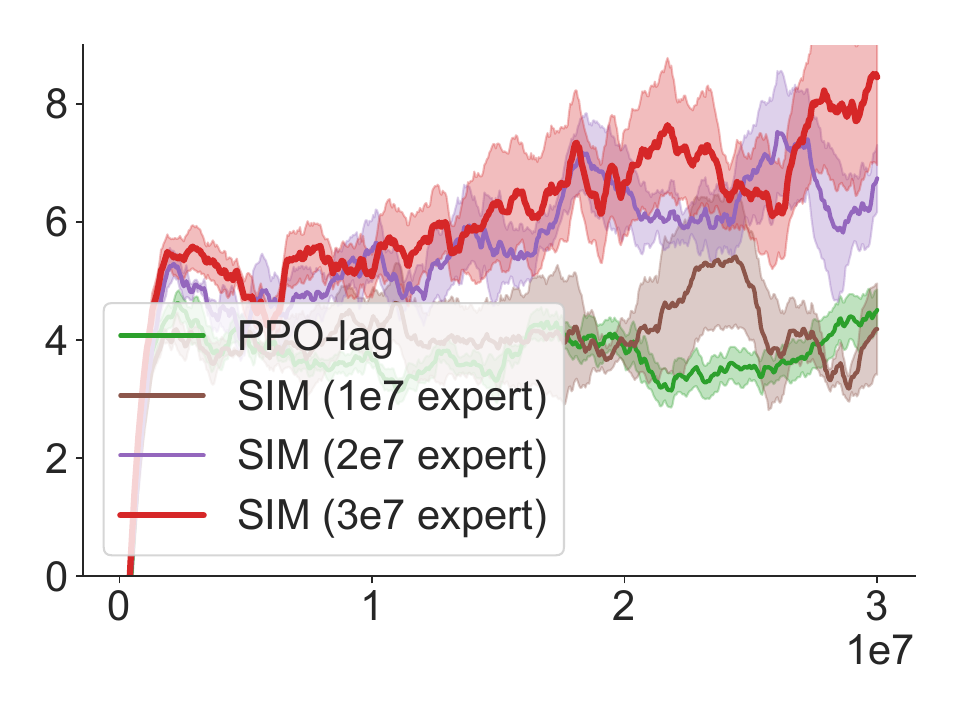}{SafetyCarButton}
\showReturn[0.21]{0}{0}{0}{0}{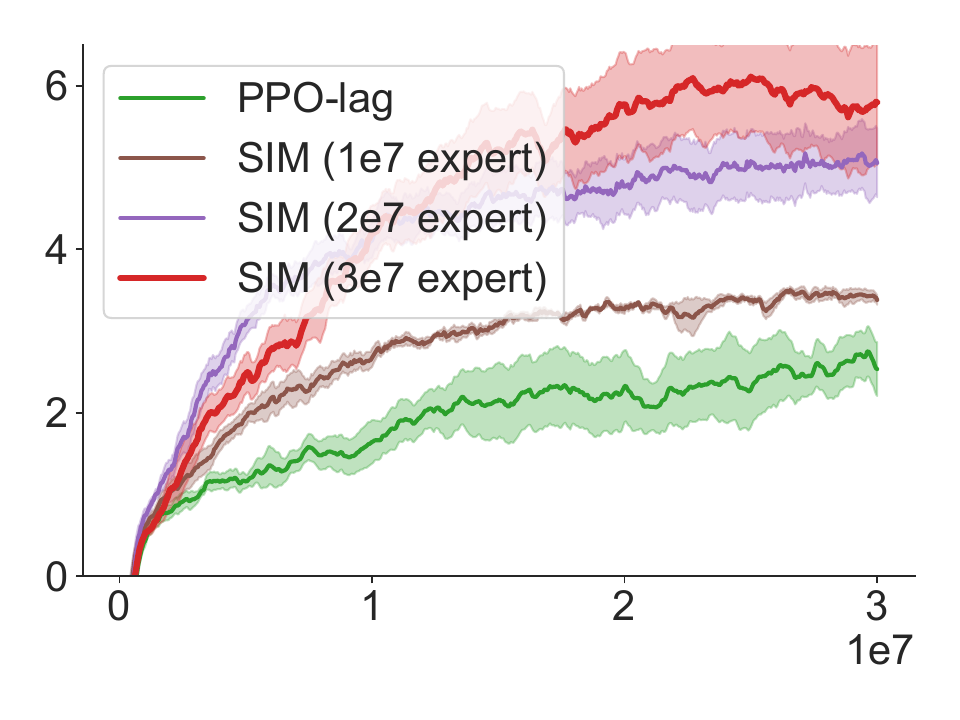}{SafetyCarPush}

\rotatebox[origin=c]{90}{\centering Cost}
\showCost[0.21]{0}{0}{0}{0}{Images/SafetyCarButton1-v0/dif_exp_Cost}
\showCost[0.21]{0}{0}{0}{0}{Images/SafetyCarPush1-v0/dif_exp_Cost}

    \caption{
Comparison results for different expertise levels of the pre-trained policy.}
    \label{fig:diff_expert_results}
\end{figure}

These results demonstrate that even without a well-trained initial policy, SIM is able to efficiently improve it and outperform the traditional PPO-Lagrangian method. 
These also indicate that, for SIM to achieve the best performance, one should start with a well-trained initial policy. As mentioned previously, SIM would greatly benefit from learning from good trajectories generated by a well-trained policy.  


\section{CONCLUSION}
We introduced a novel framework to solve Constrained RL without relying on cost estimations or cost penalties, as commonly done in prior work. Our new algorithm, based on the idea of learning to mimic the behavior of good demonstrations and avoid bad demonstrations, is non-adversarial and allows learning from demonstration sets to evolve during the training process. Extensive experiments on several challenging benchmark tasks demonstrate that our approach achieves superior performance compared to prior constrained RL algorithms.
Our IL-based framework would open new directions to address safe RL problems without explicitly considering the reward or cost function. Our algorithm relies on sets of good demonstrations generated by a pre-trained policy, so a limitation would be that our algorithm will not work if it is difficult to generate feasible trajectories due to, for instance, strict constraints. A future direction would be to develop new IL-based algorithms to address such issues.
\section*{Acknowledgment}
This research/project is supported by the National Research Foundation Singapore and DSO National Laboratories under the AI Singapore Programme (AISG Award No: AISG2-RP-2020-016).

\bibliography{aaai24}

\onecolumn
\appendix
\section{Missing Proofs}

\subsection{Proof of Lemma \ref{lm-1}}
\textbf{Lemma \ref{lm-1}:}
\textit{ For any $\lambda>0$, if there exists a policy $\pi^*$ such that $P_{\pi^*}(\tau) = 0$ for all $\tau \in \Omega^B$, and $$
     P_{\pi^*}(\tau) = \frac{P_{\pi^0}(\tau)}{  \sum_{\tau'\in \Omega^G}P_{\pi^0}(\tau')};~\forall \tau \in \Omega^G 
     $$ 
     then $\pi^*$ is an optimal policy to \eqref{BC-good-bad}.
}
\begin{proof}
To simplify the proof, let us first prove the following result:
\begin{lemma}\label{lm:lm2}
   Given $\widehat{p}_0,\widehat{p}_1,\ldots,\widehat{p}_N \in [0,1]$ such that $\sum_{n=1}^N \widehat{p}_n \leq 1$, then vector $p^*$ such that $p^*_n = \frac{\widehat{p}_n}{\sum_{n'}\widehat{p}_{n'}}$ is a unique optimal solution to  the  following optimization problem \begin{equation}\label{prob:sub-p}
      \max_{p\in[0,1]^N} \left\{f(p) = \sum_{n} \widehat{p}_n \ln {p}_n\Big|~ \sum_n p_n\leq 1  \right\} 
   \end{equation}   
\end{lemma}
\begin{proof}
We first see that the objective function $f(p)$ is strictly concave in $(0,1)^N$, implying that \eqref{prob:sub-p} always has a unique optimal solution. 
We write the Lagrange dual of \eqref{prob:sub-p} as
    \[
    \cL(p,\eta) = \sum_{n} \widehat{p}_n \ln {p}_n - \eta\left(\sum_{n} p_n- 1\right)
    \]
Let $\overline{p}$ be the optimal solution of \eqref{prob:sub-p} and $\overline{\eta}$ be its associated Lagrange multiplier. The KKT conditions imply that the following hold
\begin{equation*}
    \begin{cases}
        \frac{\partial \cL(\overline{p},\eta)}{\partial p_n} = 0,~\forall n=1,\ldots,N \\
        \eta(\sum_n \overline{p}_n) = 0
    \end{cases}   
\end{equation*}
which is equivalent to 
\[
\begin{cases}
    \frac{\widehat{p}_1}{\overline{p}_1} = \frac{\widehat{p}_2}{\overline{p}_2} = ... = \frac{\widehat{p}_N}{\overline{p}_N} = \eta \\
    \eta(\sum_n \overline{p}_n-1) =0
\end{cases}
\]
We than see that $\eta >0$, thus $\sum_n \overline{p}_n-1 = 0$. On the other hand $\eta = \frac{\sum_{n} \widehat{p}_n}{\sum_n \overline{p}_n} = \sum_{n} \widehat{p}_n,$
which implies that 
$\overline{p}_n = \frac{\widehat{p}_n}{\sum_{n'}\widehat{p}_{n'}}$. Thus $p^* = \overline{p}$ is a unique optimal solution to \eqref{prob:sub-p}, as desired. 
\end{proof}

We now get back to the main proof. Recall that the objective function of the training with good and bad trajectories, under BC, is 
\[   
 F(\pi) = \lambda \sum_{\substack{ \tau \in \Omega^G}} P_{\pi^0}(\tau) \ln P_{\pi}(\tau)  - (1-\lambda) \sum_{\substack{ \tau \in \Omega^B}} P_{\pi^0}(\tau) \phi(\ln P_{\pi}(\tau)) 
\]  
We now assume  that the regularizer $\phi(.)$ map $(-\infty,0]$ to a finite interval $[a,b]$. Since $\phi()$ is monotone, we see that, for any $\tau \in \Omega$, $P_{\pi}(\tau)\geq P_{\pi^*}(\tau)$, thus $\phi(P_{\pi}(\tau))\geq \phi(P_{\pi^*}(\tau))$. Moreover, from Lemma \ref{lm:lm2}, the first term of $F(\tau)$ can be bounded as 
\[
\sum_{\substack{ \tau \in \Omega^G}} P_{\pi^0}(\tau) \ln P_{\pi}(\tau) \leq \sum_{\substack{ \tau \in \Omega^G}} P_{\pi^0}(\tau) \ln P_{\pi^*}(\tau)  
\]
implying $F(\pi)\leq F(\pi^*)$ for any policy $\pi$. So, if $\pi^*$ exists, it will be optimal to \eqref{BC-good-bad}. 
\end{proof}

\subsection{Proof of Proposition \ref{th-1}}
\textbf{Proposition \ref{th-1}: }
\textit{    For any $\lambda>0$, let $\pi^0$ be a pre-trained  feasible policy, $\bR^E = \bbE_{\tau\sim\pi^0}\Big[R(\tau)\Big]$, and  $\Omega^B$ be a collection  of trajectories of low reward and high-cost values $$\Omega^B = \left\{\tau\Big|~ R(\tau) \leq \bR^E ,~ C(\tau) >c_{\max}\}\Big]\right\}$$
the  optimal policy mentioned in Lemma \ref{lm-1} is feasible to the cost constraint while offering a better expected reward than the pre-trained policy $\pi^0$, specifically, 
{\small\begin{align}
    \bbE_{\pi^*} \left[R(\tau)\right] - \bbE_{\pi^0} \left[R(\tau)\right] &=\frac{\sum_{\tau}{P_{\pi^0}(\tau) (\bR^E-R(\tau))}}{1-P_{\pi^0}(\Omega^B)} \geq 0 \label{eq:p1-eq11} \\
    \bbE_{\tau \sim \pi^*} \left[C(\tau)\right] &\leq c_{\max}\nonumber
\end{align}}
where $P_{\pi^0}(\Omega^B) = \sum_{\tau\in \Omega^B}P_{\pi^0}(\tau)$.}
\begin{proof}
    Recall that $P_{\pi^*}(\tau) = 0$ for all $\tau \in\Omega^B$  and $P_{\pi^*}(\tau) = \frac{P_{\pi^0}(\tau)}{\sum_{\tau'} P_{\pi^0}(\tau)}$ for all $\tau \in\Omega^G$. We write the expected reward under $\pi^*$ as 
\begin{align*}
   \bbE_{\tau\sim \pi^*}[R(\tau)] &=  \sum_{\tau\in \Omega^G} P_{\pi^*}(\tau) R(\tau) = \frac{\sum_{\tau\in \Omega^G} 
 P_{\pi^0}(\tau)}{1 - P_{\pi^0}(\Omega^B)}
\end{align*}
 Thus
 \begin{align*}
   \bbE_{\tau\sim \pi^*}[R(\tau)] - \bbE_{\tau\sim \pi^0}[R(\tau)] &=  \sum_{\tau\in \Omega^G} P_{\pi^*}(\tau) R(\tau) = \frac{\sum_{\tau\in \Omega^G} 
 P_{\pi^0}(\tau)R(\tau) - \sum_{\tau} P_{\pi^0} R(\tau) (1-P_{\pi^0}(\Omega^B))}{1 - P_{\pi^0}(\Omega^B)} \\
 &=\frac{\sum_{\tau\in \Omega^G} 
 P_{\pi^0}(\tau)R(\tau) - \sum_{\tau} P_{\pi^0} R(\tau) +P_{\pi^0}(\Omega^B) \bR^E }{1 - P_{\pi^0}(\Omega^B)} \\
 &=\frac{\sum_{\tau\in \Omega^B} 
 P_{\pi^0}(\tau)(\bR^E - R(\tau))  }{1 - P_{\pi^0}(\Omega^B)} \stackrel{(a)}{\geq} 0 \\
\end{align*}
where $(a)$ is due to the fact that  $R(\tau)\leq \bR^E$ for all $\tau \in \Omega^B$. We now consider the expected cost given by $\pi^*$. Let $\bC^E = \bbE_{\pi^0}[C(\tau)]$, we write 
\begin{align*}
    \bbE_{\pi^*}[C(\tau)] - \bC^E &= \sum_{\tau} P_{\pi^*}(\tau) C(\tau) - \bC^E  \nonumber\\
    &= \frac{\sum_{\tau\in \Omega^G} 
 P_{\pi^0}(\tau)C(\tau) -  \sum_{\tau} 
 P_{\pi^0}(\tau)C(\tau) + \bC^E P(\Omega^B)}{1 - P_{\pi^0}(\Omega^B)} \\
 &=\frac{ \bC^E P(\Omega^B) - \sum_{\tau} 
 P_{\pi^0}(\tau)C(\tau)}{1 - P_{\pi^0}(\Omega^B)} \\
 &\stackrel{(b)}{\leq}\frac{ \bC^E P(\Omega^B) - \sum_{\tau\in \Omega^B} 
 P_{\pi^0}(\tau)c_{max}}{1 - P_{\pi^0}(\Omega^B)} = \frac{ (\bC^E - c_{max}) P(\Omega^B) }{1 - P_{\pi^0}(\Omega^B)} \stackrel{(c)}{\leq} 0\\
\end{align*}
where $(b)$ is because $C(\tau) >c_{max}$ (according to the way we choose $\Omega^B$), and $\bc^E\leq c_{max}$ $(\pi^0$ is feasible w.r.t the cost constraint). So, we have $ \bbE_{\pi^*}[C(\tau)]\leq \bC^E\leq c_{max}$, implying that $\pi^*$ is safe, as desired. 
\end{proof}

\subsection{Proof of Lemma \ref{lm-2}}
\textbf{Lemma \ref{lm-2}: }
\textit{If $\lambda = 0$, any policy $\pi^*$ such that $P_{\pi^*}(\tau) = 0$ for all $\tau \in\Omega^B$ is optimal for \eqref{BC-good-bad}.  }
\begin{proof}
   This can be obviously seen, as if $\lambda = 0$, then the objective function becomes  $F(\pi) = -  \sum_{\substack{ \tau \in \Omega^B}} P_{\pi^0}(\tau) \phi(\ln P_{\pi}(\tau)) $. Since $\phi(.)$ is monotone, $F(\pi) \geq F(\pi^*)$ for any policy $\pi$. 
\end{proof}

\subsection{Proof of Proposition \ref{th-2}}
\textbf{Proposition \ref{th-2}: }
\textit{If $\lambda = 0$, and the bad set $\Omega^B$ is selected in the same manner as in Theorem \ref{th-1}, then the optimal policy $\pi^*$ from Lemma \eqref{lm-2} is feasible, but it does not necessarily provide a higher expected reward than the policy $\pi^0$.}

\begin{proof}
  \begin{figure}[htb] 
\centering
    \includegraphics[width=0.3\linewidth]{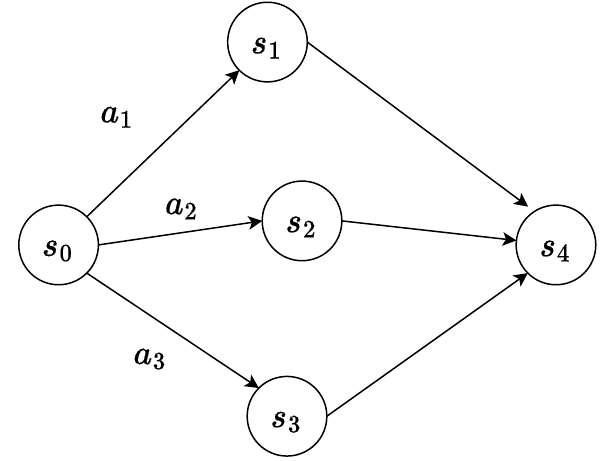} 
    \caption{Example} 
    \label{fig:SMDP2} 
\end{figure}
    According to the  Lemma \ref{lm-2}, any $\pi^*$ such that $P_{\pi^*}(\tau) = 0$ is optimal for \eqref{BC-good-bad}. To prove that $\pi^*$ may not offer a higher expected reward than $\pi^0$, we will use the  counter-example shown in Figure \ref{fig:SMDP2}. There are 5 states and the MDP is deterministic. The rewards  and cost are set as  $r(s_0) = 0,r(s_4) = 0, r(s_1) = 2, r(s_2) = 3,r(s_3) = 8$, and $d(s_0)=0, d(s_4) = 0$, $d(s_1) = 5, d(s_2) = d(s_3) = 1$. The initial policy  is set as $\pi^0(a_1|s_0) = 1/5$, $\pi^0(a_2|s_0)=1/5$ and $\pi^0(a_3|s_0) = 3/5$. We also choose $c_{max} = 2$
    The expected reward is $\bR^E = 5.6$ and and expected cost is $\cC^E = 1.8$. It is then clear that  the trajectory $\{s_0,s_1,s_4\}$ should be classified in the bad set.  Policy $\pi^*$ such that $\pi^*(a_1|s_0) = 0$, $\pi^*(a_2|s_0) = 4/5$ and $\pi^*(a_3|s_0) = 1/5$ is definitely optimal for \eqref{BC-good-bad}, according to Lemma \ref{lm-2}. We however see that $\bbE_{\pi^*}[R(\tau)] = 4 < \bR^E$, implying that $\pi^*$ offers a worse expected reward than the initial policy $\pi^0$. We complete the proof.  
\end{proof}

\subsection{Proof of Proposition \ref{th-3}}

\textbf{Proposition \ref{th-3}:}
\textit{The following hold
    \begin{itemize}
        \item[(i)] If we select the bad set as $\Omega^B = \left\{\tau\Big|~ R(\tau) \leq \bR^E ,~ C(\tau) >\bC^E\}\Big]\right\}$, then it is guaranteed that  $\pi^*$ offers a higher (or equal) expected reward and lower (or equal) expected cost, compared to those from $\pi^0$, where  $\bC^E = \bbE_{\tau\sim \pi^0} [C(\tau)]$ 
        \item[(ii)] If the pre-trained policy $\pi^0$ is not feasible, then if we select the bad set as $\Omega^B = \left\{\tau\Big| C(\tau) >c_{\max}\}\Big]\right\}$, then it is guaranteed that  $\pi^*$ is feasible 
        \item[(iii)] If the cost function is not accessible, but there is an oracle that can tell us which trajectories are violating the constraint,  then by selecting, $\Omega^B = \left\{\tau\Big| \tau \text{ is violated }\Big]\right\}$, then $\pi^*$ is feasible.
    \end{itemize}}

\begin{proof}
The proof is similar to the proof of Proposition \ref{th-1}. For \textit{(i)}, we also write 
\begin{align*}
 \bbE_{\pi^*}[R(\tau)] - \bR^E &=\frac{\sum_{\tau\in \Omega^B} 
 P_{\pi^0}(\tau)(\bR^E - R(\tau))  }{1 - P_{\pi^0}(\Omega^B)} \\
 \bbE_{\pi^*}[C(\tau)] - \bC^E &= \frac{ \sum_{\tau\in \Omega^B}(\bC^E - C(\tau)) P_{\pi^0}(\tau) }{1 - P_{\pi^0}(\Omega^B)}
\end{align*}
Then according to the way we select $\Omega^B$ in (i), we should have $ \bbE_{\pi^*}[R(\tau)] \geq \bR^E$ and $\bbE_{\pi^*}[C(\tau)] \leq \bC^E$, implying that $\pi^*$ yields a higher expected reward and lower expected cost, compared to $\pi^0$. 

For (ii), since $C(\tau)>c_{max}$ for all $\tau \in \Omega^B$, $C(\tau)\leq c_{max}$ for all $\tau \in \Omega^G$. We
 write the expected cost under $\pi^*$ as 
\begin{align*}
     \bbE_{\pi^*}[C(\tau)] &= \frac{\sum_{\tau\in \Omega^G} 
 P_{\pi^0}(\tau)C(\tau)}{1 - P_{\pi^0}(\Omega^B)} \leq \frac{\sum_{\tau\in \Omega^G} 
 P_{\pi^0}(\tau)c_{max}}{1 - P_{\pi^0}(\Omega^B)} = c_{max} 
\end{align*}
So, $\pi^*$ is safe.

Claim \textit{(iii)} is the same as \textit{(ii)}, in the sense that the oracle can correctly select the bad set $\Omega^B = \{\tau|~ C(\tau)>c_{max}\}$. Thus, the policy $\pi^*$, if exits, will be safe. 
\end{proof}

\subsection{Proof of Proposition \ref{prop:DM}}
\textbf{Proposition \ref{prop:DM}:}\textit{
    The maximization in \eqref{eq:DM-GD-eq1} is achieved at $K^*(s,a)$ such that  
    \[
    \ln \left(\frac{K^*(a,s)}{1-K^*(s,a)}\right) = \ln\frac{\rho^{B}(s,a)}{\rho^{G,\pi}(s,a)}
    \]}

\begin{proof}
We first look at the trainning objective  of $K(s,a)$ and write
\begin{align}
    J(K,\pi)&=\bbE_{\rho^B}[\ln(K(s,a))]+\frac{1}{2}\bbE_{\rho^{\pi}}[\ln(1-K(s,a))]+\frac{1}{2}\bbE_{\rho^{G}}[\ln(1-K(s,a))] \nonumber\\
    &=\bbE_{\rho^B}[\ln(K(s,a))] + \sum_{(s,a)} \ln(K(s,a)) \frac{\rho^\pi(s,a) +\rho^G(s,a)}{2} \nonumber \\
    &=\bbE_{\rho^B}[\ln(K(s,a))]+\bbE_{\rho^{\pi,G}}[\ln(1-K(s,a))]\nonumber\\
    &=\sum_{(s,a)} \ln(K(s,a)) \rho^B(s,a)  + \ln(1-K(s,a)) \rho^{\pi,G}(s,a) 
\end{align}
So, to maximize $J(K,\pi)$, each component $\ln(K(s,a)) \rho^B(s,a)  + \ln(1-K(s,a)) \rho^{\pi,G}(s,a) $ needs to be maximized. To study this maximization problem,  we  consider the following simple optimization problem $\max_{x\in (0,1)}\{f(x) = \ln(x) a + \ln (1-x) b\}$, where $a,b \geq 0$. We first see that $f'(x) = \frac{a}{x} - \frac{b}{1-x}$. Thus if we set $f'(x) = 0$, this equation has a unique solution as $x^* = \frac{a}{a+b}$. Moreover $f'(x)\leq 0$ if $x\leq x^*$  and $f'(x)\geq 0$ if $x\geq x^*$, thus $x^*$ is a unique solution to $\max_{x\in (0,1)}\{f(x) = \ln(x) a + \ln (1-x) b\}$.

We now get back to the maximization 
\begin{equation}\label{eq:eq1}
\max_{K} \left\{\ln(K(s,a)) \rho^B(s,a)  + \ln(1-K(s,a)) \rho^{\pi,G}(s,a) \right\}    
\end{equation}
From the above small problem, we know that \eqref{eq:eq1} has a unique optimization solution $K^*(s,a)$ such that 
\[
K^*(s,a) = \frac{\rho^B(s,a)}{\rho^B(s,a) + \rho^{\pi,G}(s,a)}
\]
implying
\[
\frac{K^*(s,a)}{1-K^*(s,a)} = \frac{\rho^B(s,a)}{\rho^{\pi,G}(s,a)}.  
\]
as desired. 
\end{proof}

\section{Additional Details}
\subsection{Method Overview}
In Figure \ref{fig:method_overview} we show a diagram  illustrating  in detail our algorithm SIM. 
\begin{figure}[htbp] 
\centering
    \includegraphics[width=1.0\linewidth]{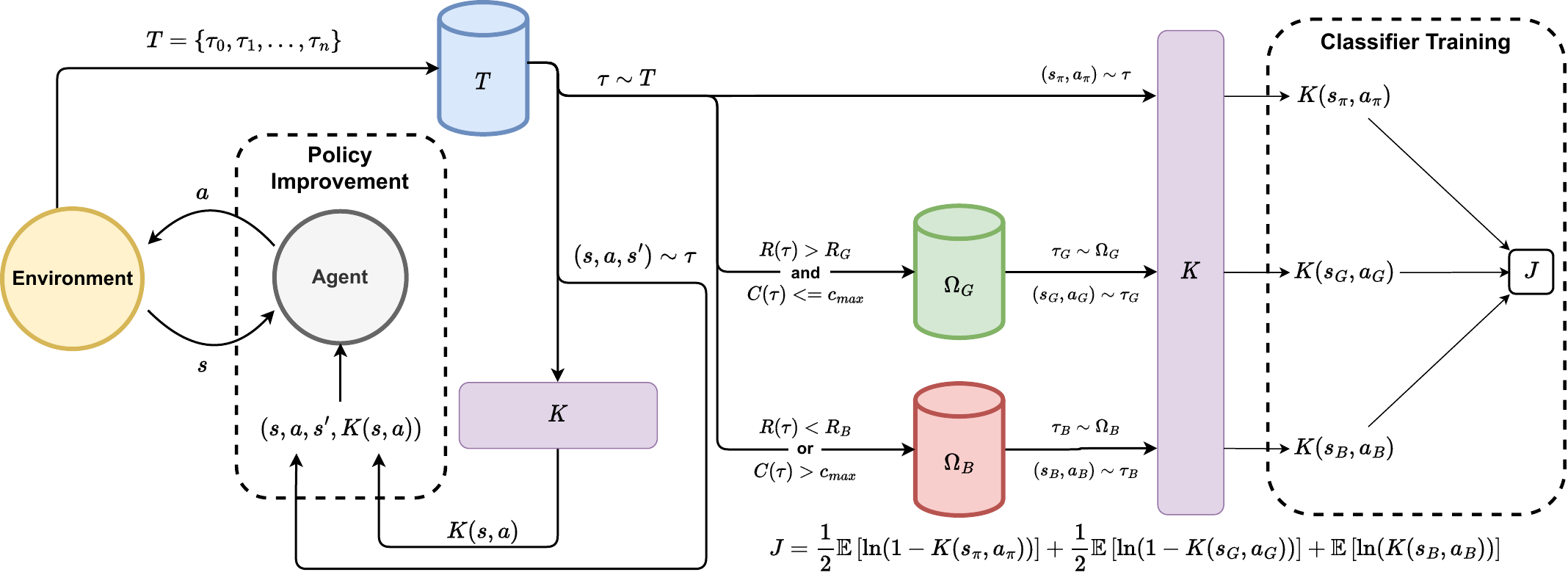} 
    \caption{Overview of SIM} 
    \label{fig:method_overview} 
\end{figure}


\subsection{Additional Settings for the Experiment with Varying Expertise Level}
We provide additional details for Section \ref{sec:varying E} (Varying Expertise Level) in the main paper. Table  \ref{tab:Expert_info} shows the expected rewards and the chosen  thresholds $R^G$ (those for selecting the good trajectories) of the three levels of $\pi^0$. 
\begin{table}[ht]
    \centering
    \begin{tabular}{|p{1cm}|p{1.2cm}|p{1.2cm}|p{1.2cm}|p{1.2cm}|}
        \hline
        \multicolumn{1}{|c|}{} & \multicolumn{2}{c|}{SafetyCarButton} & \multicolumn{2}{c|}{SafetyCarPush} \\
        \hline
        Steps & Expected return & $R_G$ & Expected return & $R_G$ \\
        \hline
        1e7 & 5.3 & 7.0 & 2.9 & 3.0 \\
        \hline
        2e7 & 8.92 & 9.0 & 5.07 & 5.0 \\
        \hline
        3e7 & 14.4 & 15.0 & 6.85 & 8.0 \\
        \hline
    \end{tabular}
    \caption{Expected rewards and thresholds  $R_G$ of for different expertise levels of $\pi^0$.}
    \label{tab:Expert_info}
\end{table}

\subsection{Relaxed  Constraints}
We provide a more detailed explanation of why relaxing the constraints is beneficial for the training of $\pi^0$.
In practical scenarios, enforcing strict constraints on trajectory generation may hinder the achievement of good trajectories due to exploration challenges and limitations in obtaining high rewards. Conversely, adopting a more relaxed constraint (constraint with higher $c_{max}$) setting could lead to higher returns, but it might also reduce the chances of satisfying the strict constraint. To address this, we initiate the training process with relaxed constraints (i.e., higher $c_{max}$) that allow us to generate a better set of  good trajectories (Figure~\ref{fig:Relaxed_expert}  illustrate an advantage of using relaxed-constrained 
 initial policy).
 Our experiments clearly demonstrate the significant advantages of employing relaxed constraints on the algorithm's final performance.
\begin{figure}[htbp] 
\centering
    \includegraphics[width=0.4\linewidth]{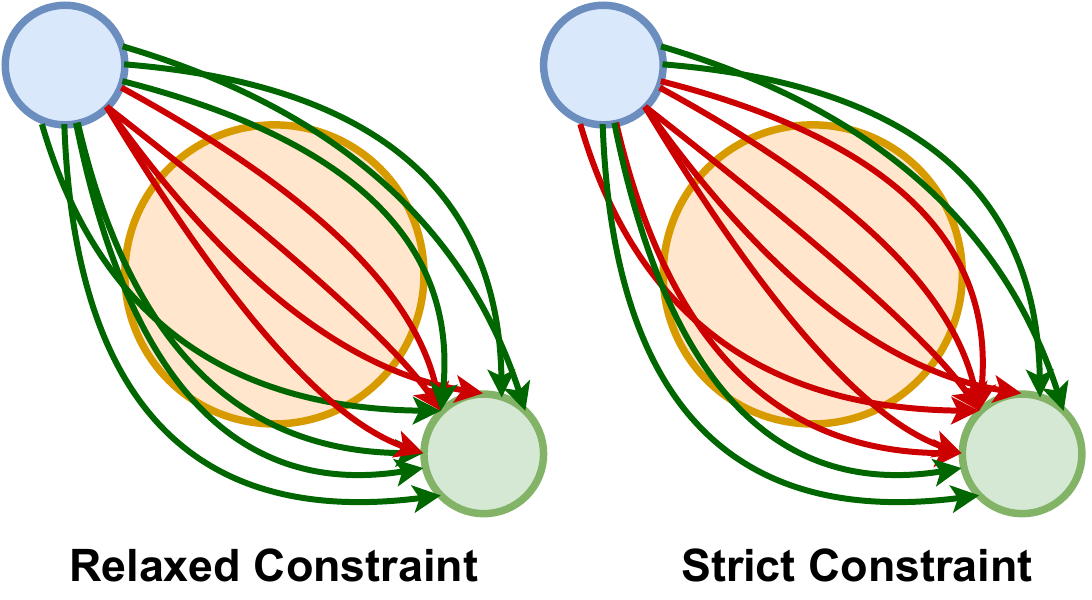} 
    \caption{Although a significant number of trajectories do not satisfy the constraints (red lines), the relaxed-constraint setting is still able to offer a considerable number of good trajectories (green lines).} 
    \label{fig:Relaxed_expert} 
\end{figure}

\subsection{Environmental Details}

\subsubsection{Safety-gym}
The Safety-gym benchmark~\citep{ray2019benchmarking}, has emerged as a highly challenging benchmark for Constraint RL. Previous research mostly focused on the easiest environment, \textit{SafetyPointGoal}, with some providing results for even simpler variations ~\cite{yang2021wcsac}. In contrast, we conducted comprehensive experiments, exploring all six challenging environments within this benchmark. These environments are illustrated in  Figure~\ref{fig:safety-gym-images} below.
\begin{figure*}[htbp]
\centering
\showEnv[0.2]{0}{0}{0}{0}{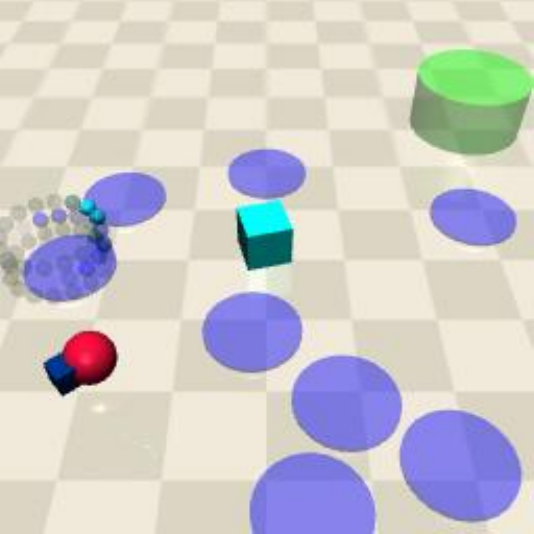}{SafetyPointGoal}
\showEnv[0.2]{0}{0}{0}{0}{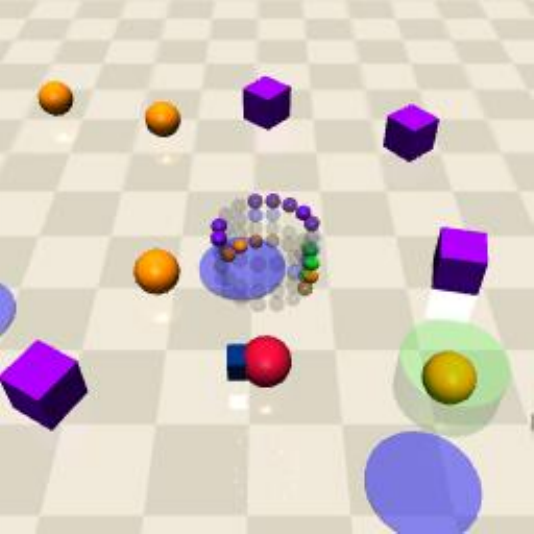}{SafetyPointButton}
\showEnv[0.2]{0}{0}{0}{0}{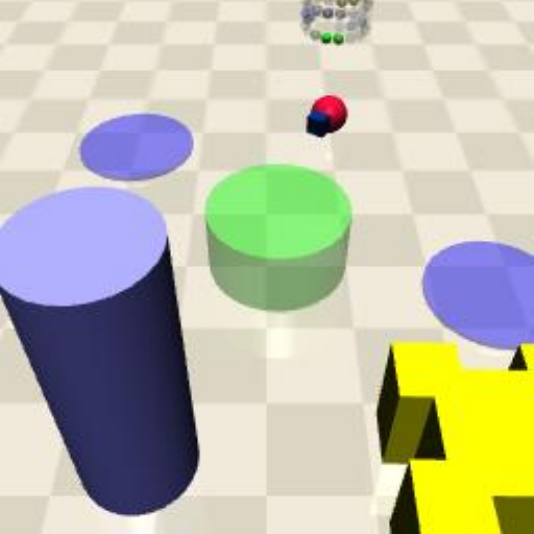}{SafetyPointPush}

\showEnv[0.2]{0}{0}{0}{0}{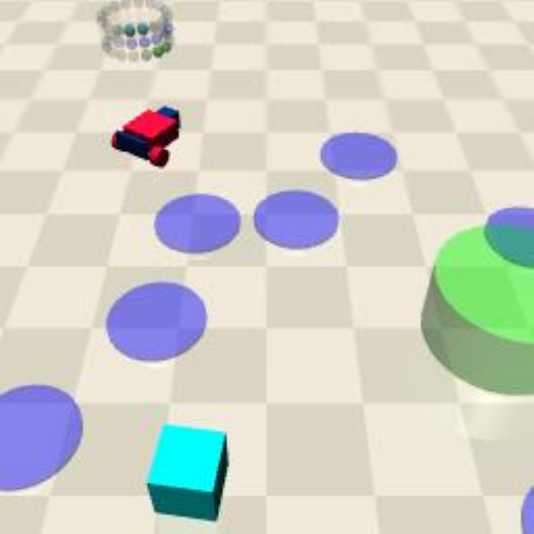}{SafetyCarGoal}
\showEnv[0.2]{0}{0}{0}{0}{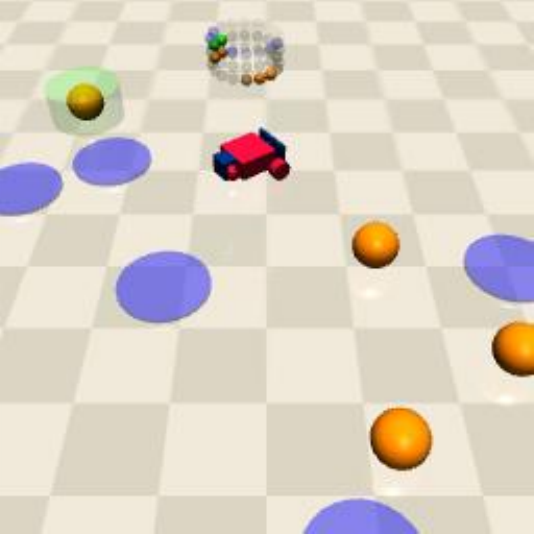}{SafetyCarButton}
\showEnv[0.2]{0}{0}{0}{0}{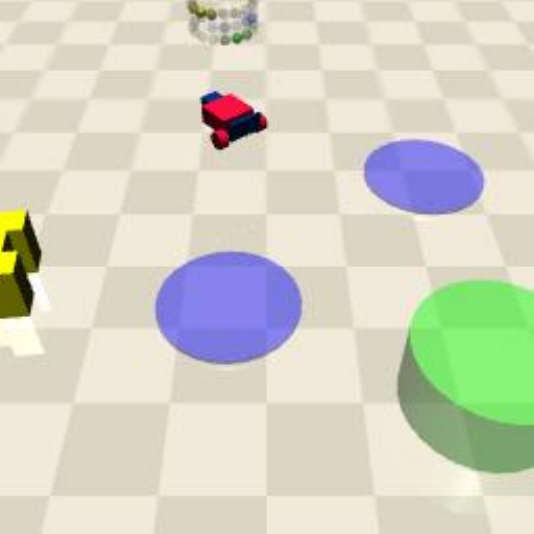}{SafetyCarPush}
    \caption{Six different environments in Safety-Gym.}
    \label{fig:safety-gym-images}
\end{figure*}

In the first pair of environments, \textit{SafetyPointGoal} and \textit{SafetyCarGoal}, the agent's primary objective is to reach the designated goal position, represented by the green area in the visuals. This must be accomplished with skillful navigation to avoid both hazardous areas (blue regions) and obstacles (cyan blocks). The \textit{SafetyPointGoal} task features a point agent, which is relatively easier to control, allowing for efficient training. On the other hand, the \textit{SafetyCarGoal} task poses a greater challenge due to the more demanding control requirements of the car agent.

Moving on to the next set of environments, \textit{SafetyPointButton} and \textit{SafetyCarButton}, the agent encounters a fresh set of challenges. In \textit{SafetyPointButton}, the primary goal is to navigate to the correct button, indicated by the green button, while carefully avoiding incorrect buttons, hazardous areas (blue regions), and maneuvering around moving obstacles (purple blocks). The \textit{SafetyCarButton} environment shares a similar objective, but with the removal of moving obstacles to reduce training difficulty. Despite this adjustment, controlling the car agent remains challenging.

Lastly, in the last pair of environments, \textit{SafetyPointPush} and \textit{SafetyCarPush}, the agent's main task is to push the yellow block to the goal area (green region) while skillfully evading hazard areas (blue regions) and the blocking pillar (dark-blue cylinder) to increase the task difficulty. Similar to the button tasks, the pillar is removed to ease the task difficulty for the car agent. 

\subsubsection{Mujoco Circle}
The \textit{Mujoco Circle} task was developed by 
 \cite{achiam2017constrained}, involving  agents moving along  a circle centered at the origin. However, there is a constraint that the agent must remain in a area within a safety region, which is smaller than the radius of the circle and represented by the green area. To further challenge the agent, two walls are introduced that hinder its ability to move freely.
 Compared to the \textit{Safety-Gym} environments, these tasks are considered less difficult because there is no randomness in the constraints imposed on the agent. The constraints are well-defined and consistent throughout the task.
 
\begin{figure*}[t]
\centering
\showEnv[0.25]{0}{0}{0}{0}{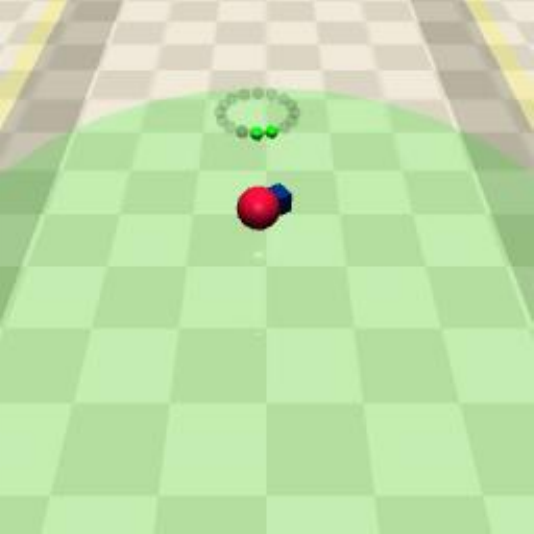}{SafetyPointCircle}
\showEnv[0.25]{0}{0}{0}{0}{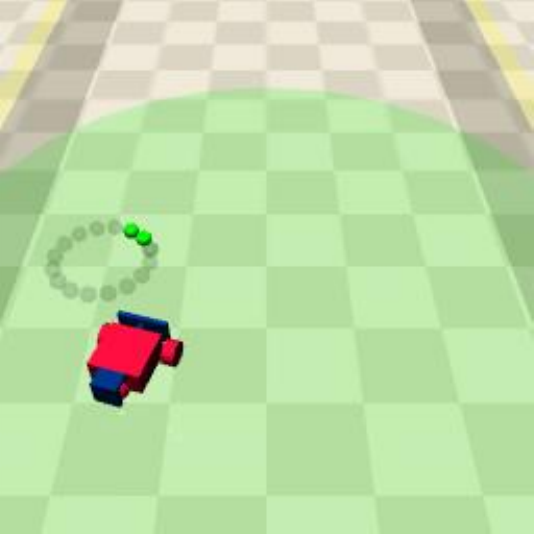}{SafetyCarCircle}
    \caption{Mujoco Circle}
    \label{fig:mujoco-circle-images}
\end{figure*}

To evaluate the performance of different agents under increasing difficulty, two types of agents are tested: Point and Car. Each agent faces the same task but with varying degrees of complexity. The Point agent is presumably the easier to control, while the Car agent poses a higher level of difficulty due to more demanding control requirements. The illustration is in Figure~\ref{fig:mujoco-circle-images}.

By conducting experiments with these agents in the \textit{Mujoco Circle} task, we can gain valuable insights into the agents' abilities to navigate the circular environment while adhering to the constraints, allowing for a comparative analysis of their performance under increasing difficulty levels.

\subsubsection{Mujoco-velocity}
\begin{figure*}[htb]
\centering
\showEnv[0.25]{0}{0}{0}{0}{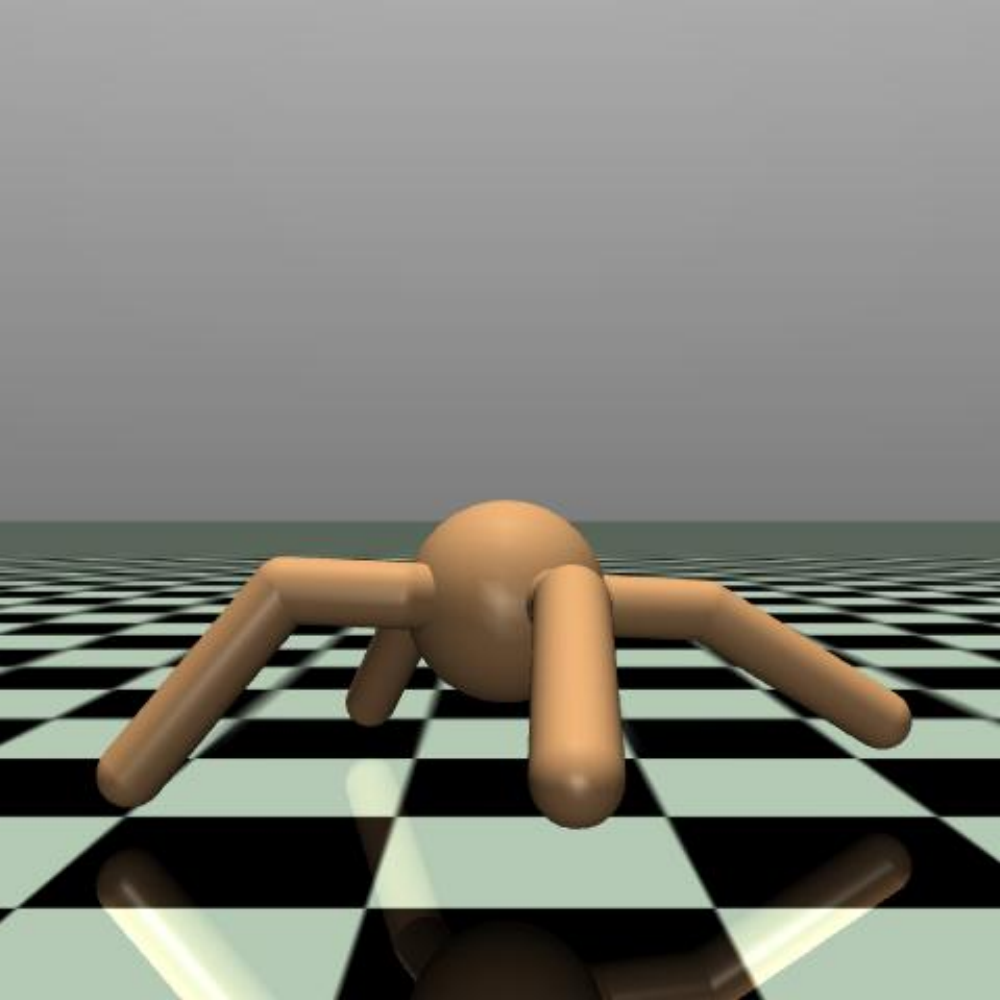}{SafetyAntVelocity}
\showEnv[0.25]{0}{0}{0}{0}{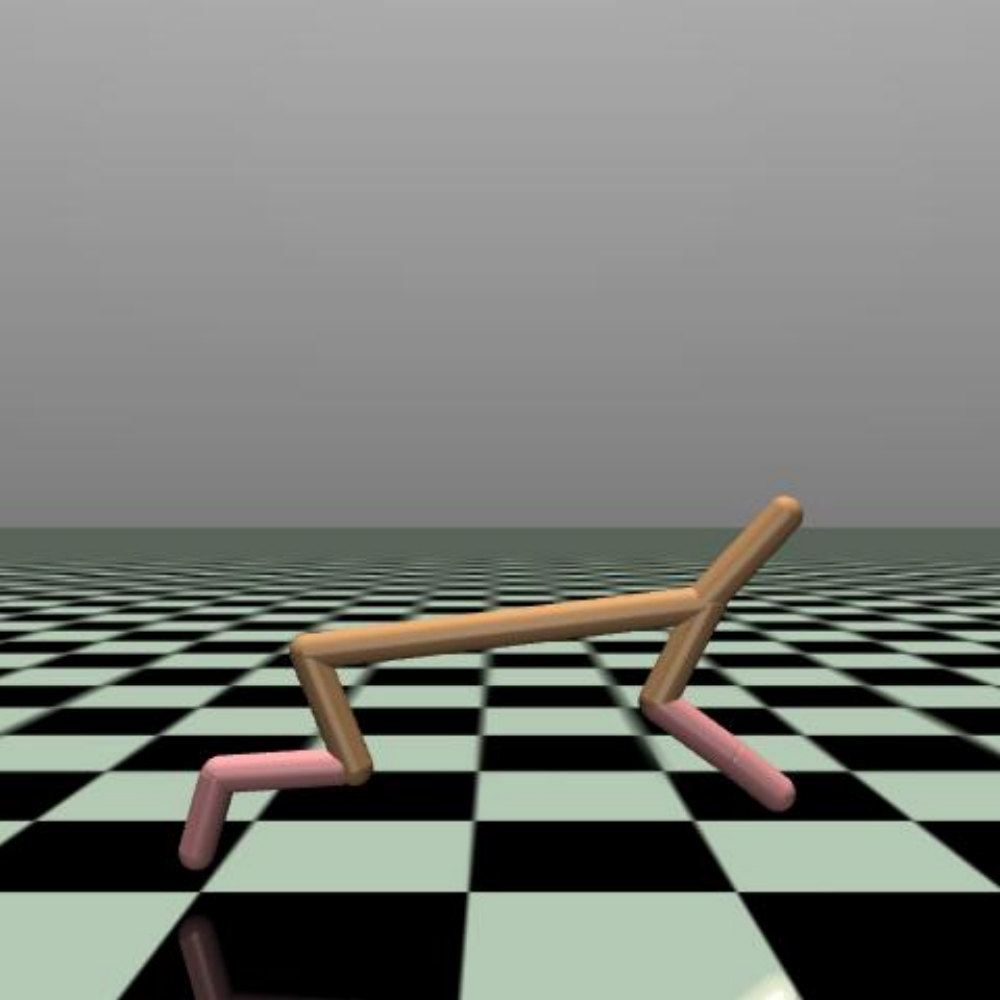}{SafetyHalfCheetahVelocity}
    \caption{Mujoco Velocity}
    \label{fig:mujoco-velocity-images}
\end{figure*}

We also test our algorithm with the \textit{Mujoco Velocity} domains. MuJoCo is an advanced framework specialized in simulating intricate physical systems that feature multi-joint mechanisms and interactions. A key aspect of MuJoCo's capabilities involves its integration of velocity constraints. In our experiments, these constraints play a crucial role as we impose specific velocity limits on the agent's movements. This action allows us to exert significant control over the motion of articulated entities within the simulation, effectively replicating real-world constraints and behaviors. The illustration is in Figure~\ref{fig:mujoco-velocity-images}.

It's worth noting that in our experimental setups, the MuJoCo environments that emphasize velocity demonstrate a relatively lower level of challenge due to the absence of external obstacles and random elements. Furthermore, achieving a high performance score doesn't solely rely on achieving high velocity. As a result, all algorithms tested within this context exhibit impressive learning capabilities.

\section{Additional Experiments}
{In this section, we provide experiments to answer 5 additional questions:}
{
\begin{itemize}
    \item[(Q5)] Can SIM provide a high-reward and safe policy using a relaxed-constraint expert?
    \item[(Q6)] What happens if the cost function is inaccessible? 
    \item[(Q7)] Would an unconstrained problem benefit from our approach?
    \item[(Q8)]  Would our approach work with CVaR constrained problems~\cite{yang2021wcsac}?
    \item[(Q9)] Do the number of initial good trajectories impact to the final performance?
\end{itemize}
}

\subsection{Hyper-parameter selection}
We conducted all experiments on a total of 4 NVIDIA RTX A5000 GPUs and 96 core CPUs. The detailed hyper-parameters are reported in Table~\ref{tab:safety_gym_params}.
\begin{table}[h]
    \centering
    \begin{tabular}{|c|c|c|c|}
        \hline
        Hyper Parameter & Safety-gym & Mujoco-circle & Mujoco-velocity \\
        \hline
        Actor Network & $[256, 256, 256]$ & $[256, 256, 256]$ & $[64,64]$\\
        \hline
        Critic Network & $[256, 256, 256]$& $[256, 256, 256]$ & $[64,64]$ \\
        \hline
        Cost Critic Network & $[256, 256, 256]$& $[256, 256, 256]$ & $[64,64]$ \\
        \hline
        Classifier Network& $[100,100,100]$& $[100,100]$ & $[100,100]$ \\
        \hline
        Gamma & $0.99$& $0.99$& $0.99$ \\
        \hline
        lr actor & $0.0001$& $0.0001$& $0.0003$ \\
        \hline
        lr Critic & $0.0001$ & $0.0001$ & $0.0001$ \\
        \hline
        lr Cost Critic & $0.0001$ & $0.0001$ & $0.0001$ \\
        \hline
        lr Classifier & $0.01$ & $0.01$ & $0.01$ \\
        \hline
        lr Penalty & $0.01$ & $0.01$ & $0.01$ \\
        \hline
        max KL & $0.05$ & $0.05$ & $0.2$ \\
        \hline
        max iteration per update & $80$ & $80$ & $120$ \\
        \hline
        buffer size & $50,000$ & $50,000$ & $20,000$ \\
        \hline
        max episode length & $1,000$ & $500$ & $1,000$ \\
        \hline
        Classifier batch size & $4,096$ & $4,096$ & $4,096$ \\
        \hline
        $R_G$ (fixed) & $\bbE_{\tau\sim\pi^E}[R(\tau)]$ & $\bbE_{\tau\sim\pi^E}[R(\tau)]$ & $\bbE_{\tau\sim\pi^E}[R(\tau)]$ \\
        \hline
        max $R_B$ & $\max(R_G/2,R_G-5.0)$ & $\max(R_G/2,R_G-10.0)$ & $\max(R_G/2,R_G-1000.0)$ \\
        \hline
    \end{tabular}
    \caption{Hyper parameters.}
    \label{tab:safety_gym_params}
\end{table}

{Moreover, to enhance stability, we use a  Chi-square function $\phi(x) = x - \frac{1}{a}x^2$ to regularize the loss function in \eqref{eq:learn_classifer} :}

{\small
\begin{align}
    &\max_{K:S\times A\rightarrow (0,1)} \Big\{J(K,\pi):=\bbE_{\rho^B}[-K(s,a)-\frac{1}{a}K(s,a)^2]\nonumber\\
    &+\frac{1}{2}\bbE_{\rho^{\pi}}[K(s,a)-\frac{1}{a}K(s,a)^2]+\frac{1}{2}\bbE_{\rho^{G}}[K(s,a)-\frac{1}{a}K(s,a)^2]  \Big\}
\end{align}}

\subsection{Training Curves of BC and BC-GB}
\begin{figure}[htbp]
\centering
\rotatebox[origin=c]{90}{\centering Return}
\showReturn[0.3]{0}{0}{0}{0}{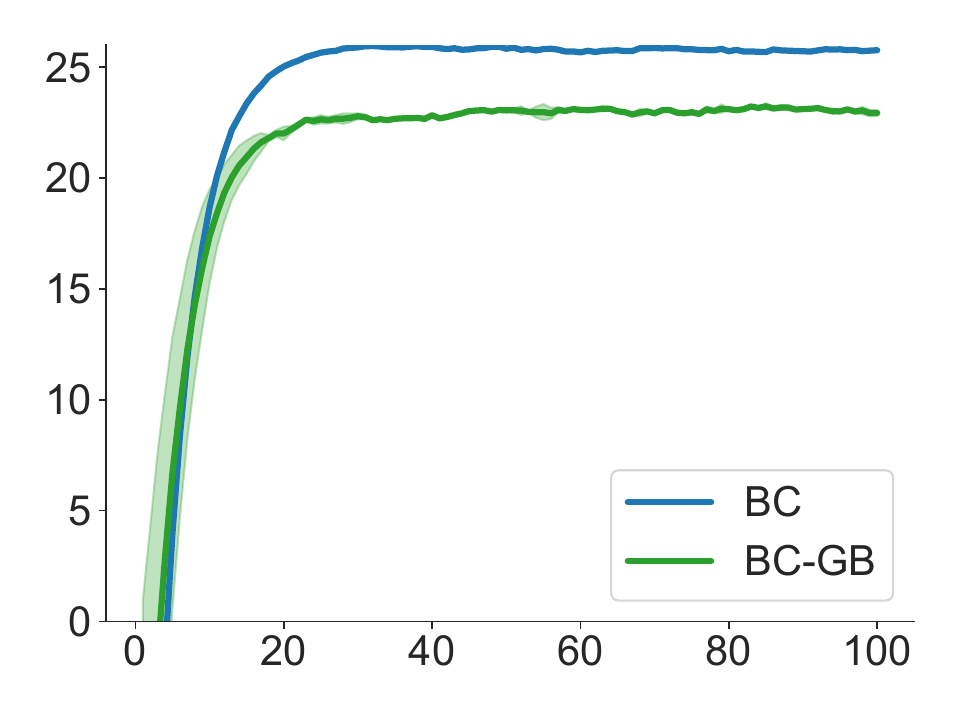}{SafetyPointGoal}
\showReturn[0.3]{0}{0}{0}{0}{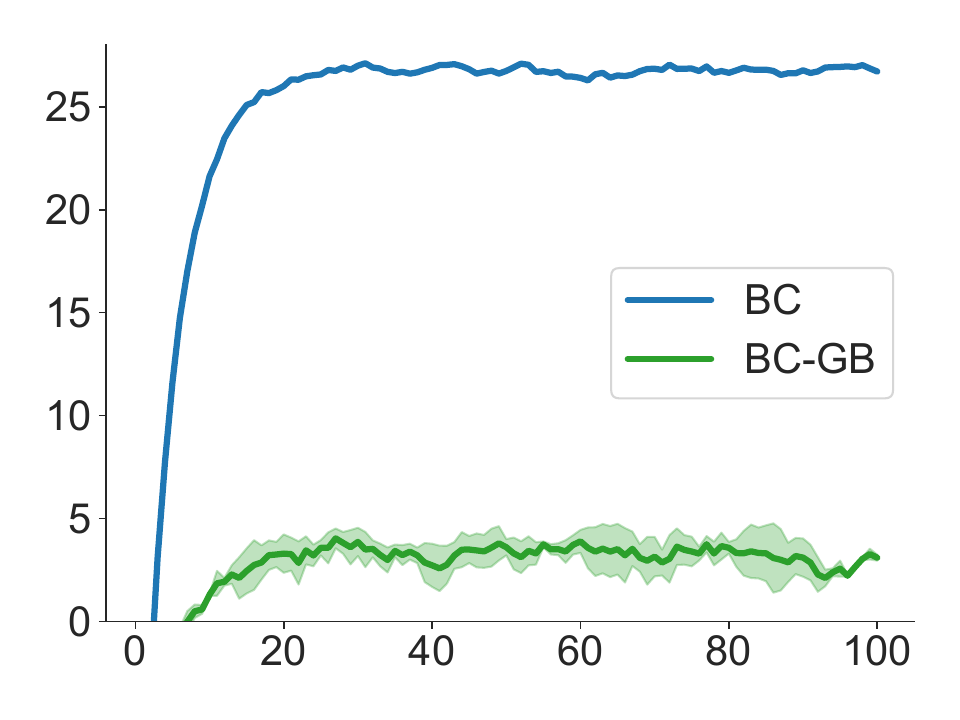}{SafetyCarGoal}

\rotatebox[origin=c]{90}{\centering Cost}
\showCost[0.3]{0}{0}{0}{0}{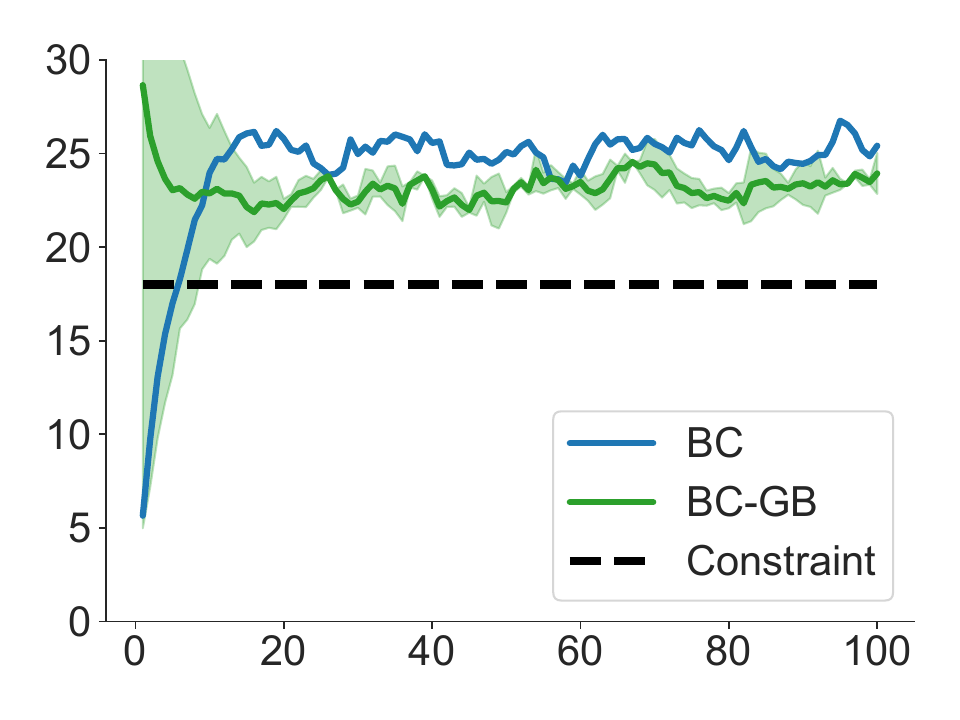}
\showCost[0.3]{0}{0}{0}{0}{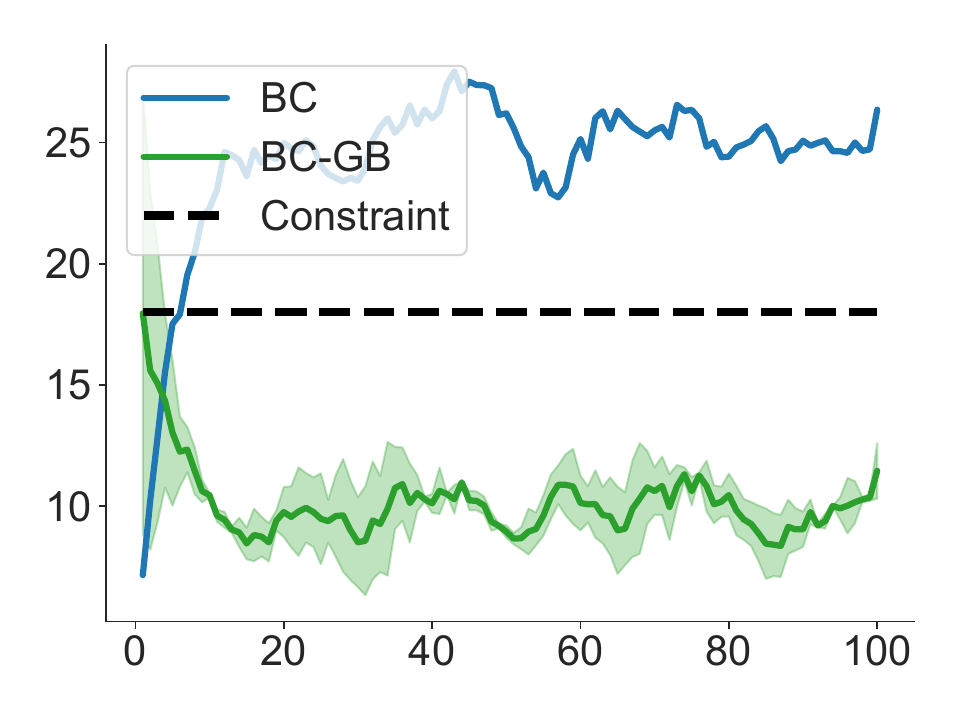}

    \caption{Training Curves of BC and BC-GB }
    \label{fig:BC_train_curves}
\end{figure}
Figure \ref{fig:BC_train_curves} shows the training curves of the BC and BC-GB approaches. This supplements our experimental results in Section \ref{sec:experiments - BCGD}, where we compare SIM against BC-based approaches.

\subsection{Unknown Cost }
\begin{figure}[htbp]
\centering
\rotatebox[origin=c]{90}{\centering Return}
\showReturn[0.3]{0}{0}{0}{0}{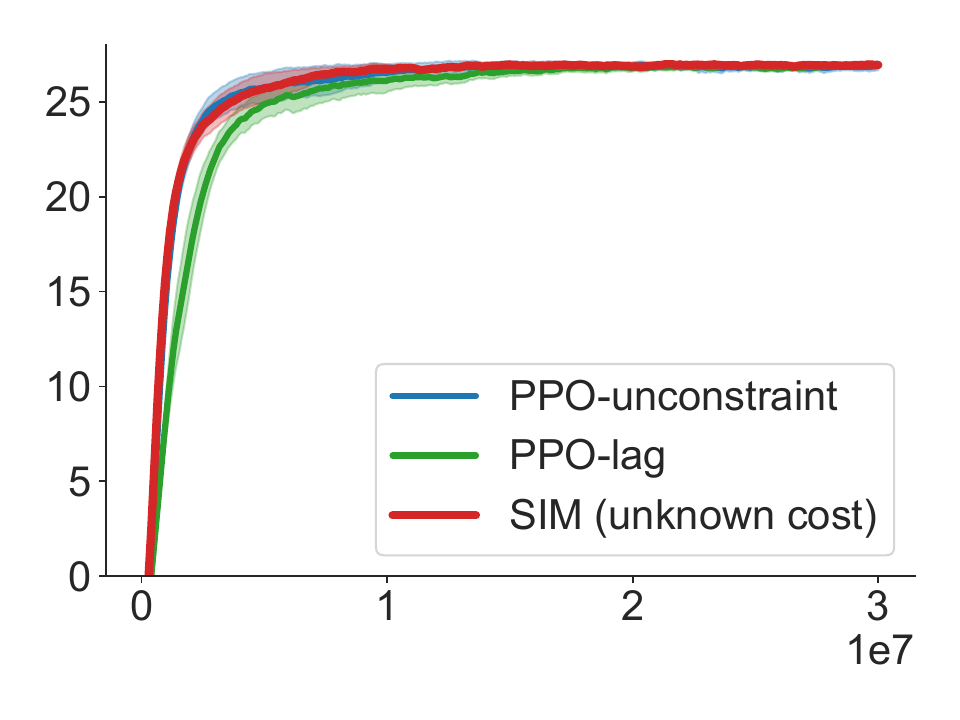}{SafetyPointGoal}
\showReturn[0.3]{0}{0}{0}{0}{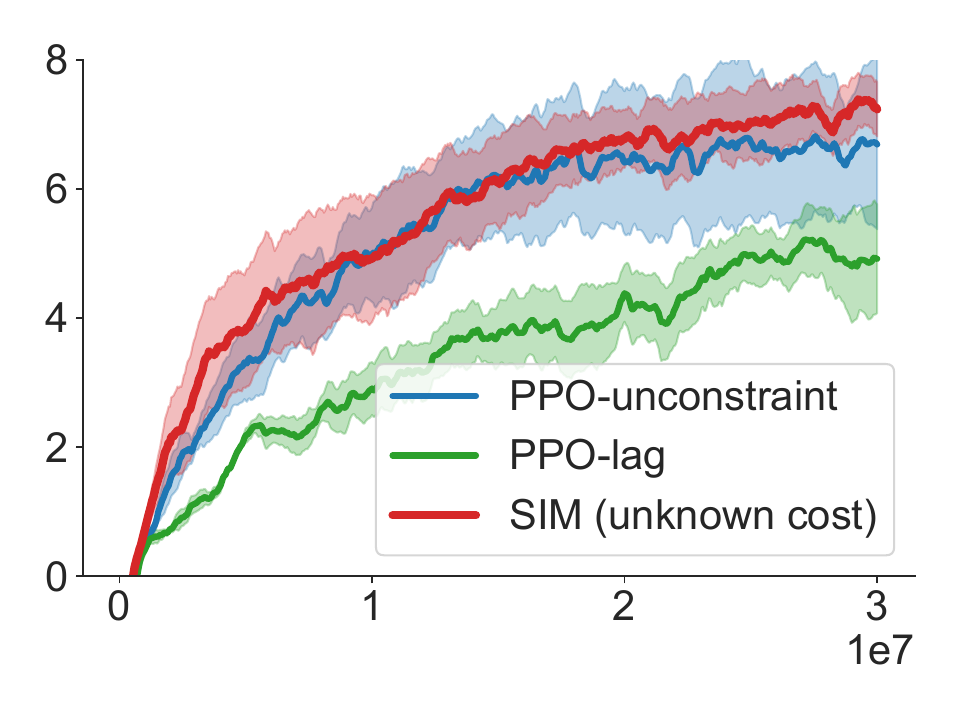}{SafetyCarPush}

\rotatebox[origin=c]{90}{\centering Cost}
\showCost[0.3]{0}{0}{0}{0}{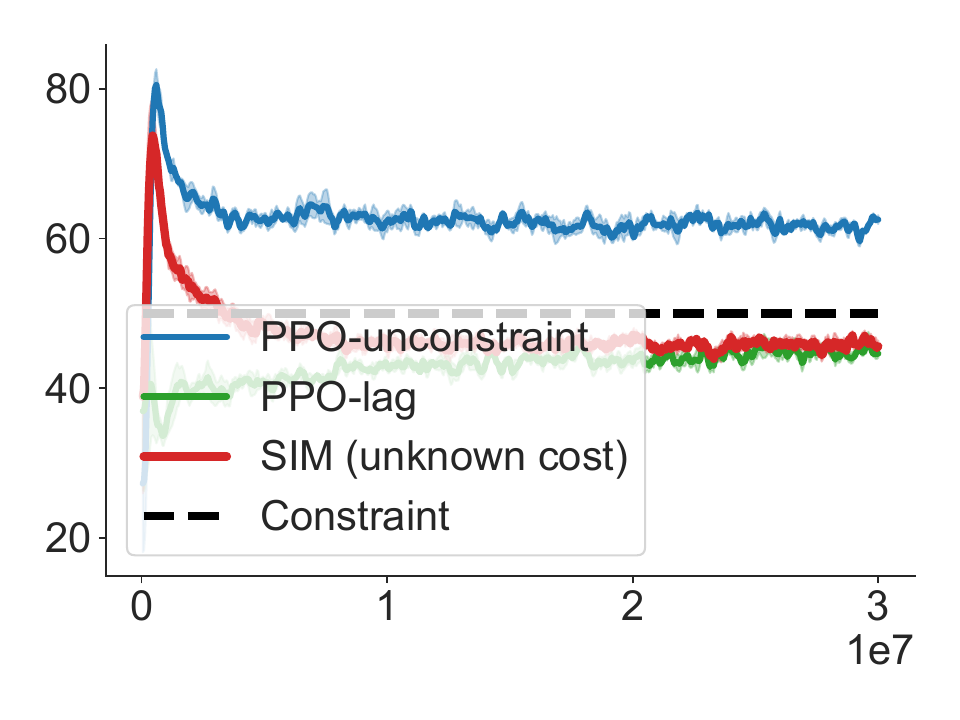}
\showCost[0.3]{0}{0}{0}{0}{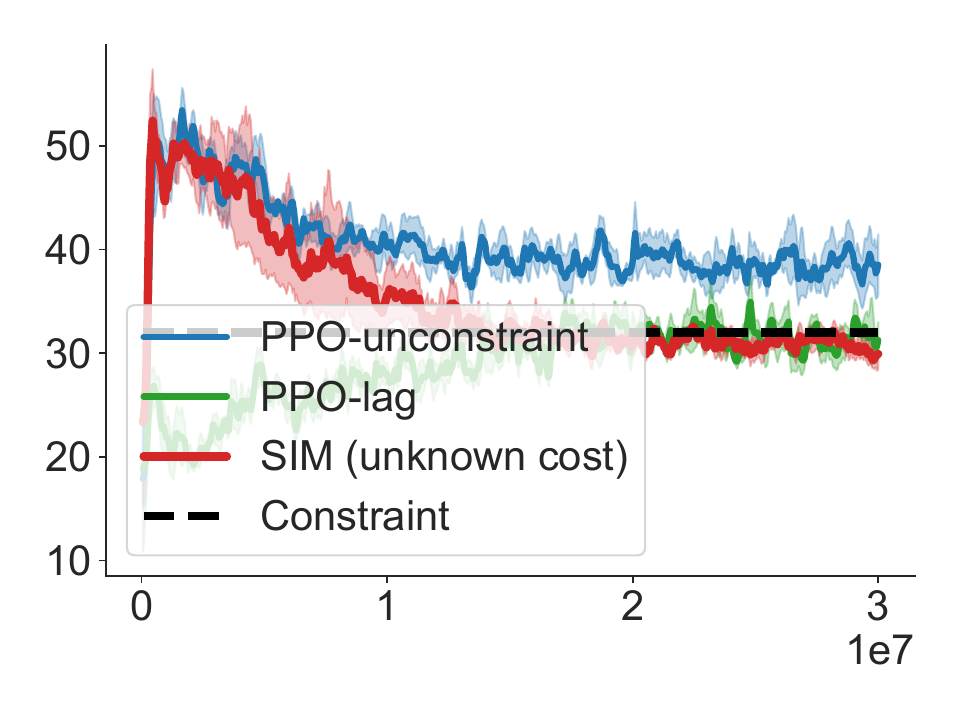}
    \caption{Results for the unknown-cost scenario.}
    \label{fig:UKC_results}
\end{figure}
To answer \textbf{Q6} (what happens if the cost function is inaccessible?), we demonstrate the capability of our method in handling the situation that the cost function is unknown. In this setting, we assume that there is an oracle telling us which trajectories are violated (i.e., the  accumulated cost is greater than $c_{max}$). In this scenario, other constrained RL algorithms do not apply, as they all rely on the cost function. On the other hand, our algorithm utilizes the identification of good and bad trajectories. Hence, it can be employed directly with the oracle's assistance. However, it's worth noting that in this situation, the oracle only aids in identifying bad trajectories, and the accessibility to good trajectories might be less potent compared to scenarios with a known cost function.
We test our method  on two Safety-Gym environments: \textit{SafetyPointGoal} and 
\textit{SafetyCarPush}. Since other constrained RL algorithms can  be used, we  just compare our algorithm with \textit{PPO-unconstraint}  and \textit{PPO-Lag}, where the later still works with the cost function. We use \textit{PPO-unconstraint} to train the initial policy $\pi^0$. 
The results shown in Figure~\ref{fig:UKC_results} indicate that SIM is able to give safe policies while offering competitive expected rewards.

The ability to work with an unknown-cost setting to improve the safety of unconstrained policies would be valuable  in may real-life situations where costs might be difficult or even impossible to get. This
enhanced adaptability opens up new opportunities for applying RL in real-world settings.

\subsection{Enhancing Unconstrained Agent}
\begin{figure}[htbp]
\centering
\rotatebox[origin=c]{90}{\centering Return}
\showReturn[0.3]{0}{0}{0}{0}{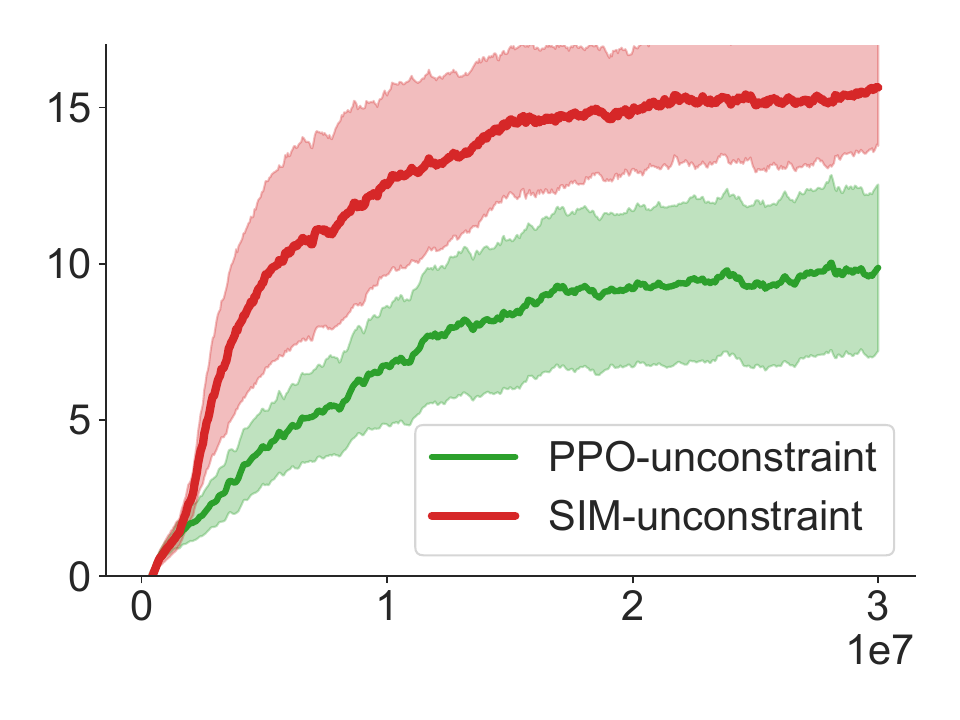}{SafetyPointPush}
\showReturn[0.3]{0}{0}{0}{0}{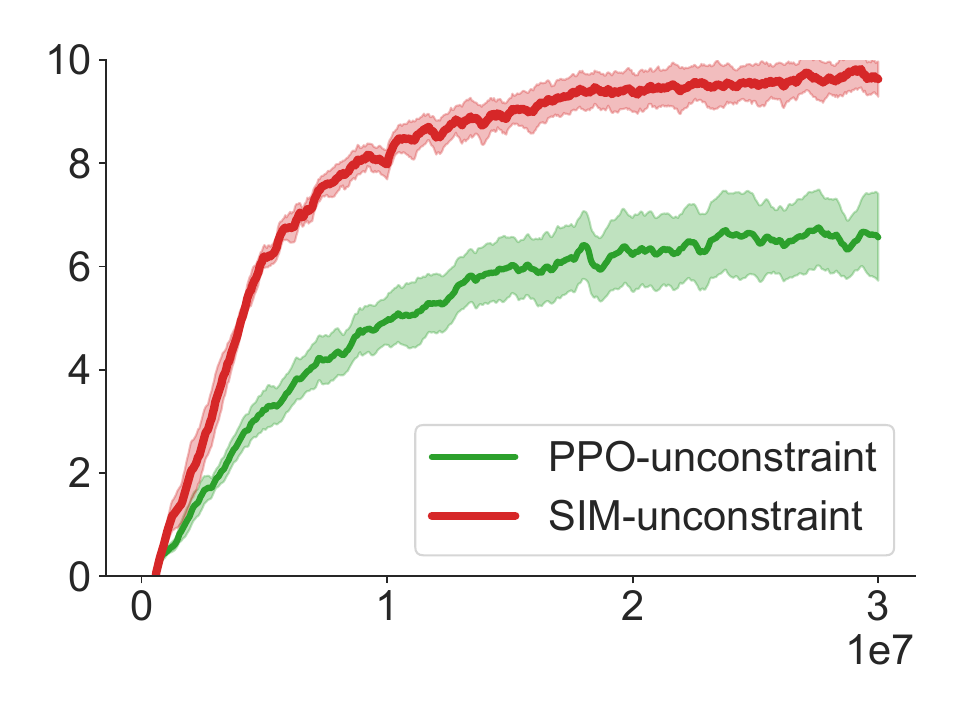}{SafetyCarPush}
    \caption{Comparison results for unconstrained tasks.}
    \label{fig:unconstrained_GB}
\end{figure}

In this experiment, we want to answer \textbf{Q7} (would an unconstrained problem benefit from our approach?). We aim to see if SIM can improve the quality of a policy trained by an unconstrained RL algorithm (e.g., \textit{PPO-unconstraint} ). To this end, we choose two Safety-Gym enviroments \textit{SafetyPointPush and SafetyCarPush} and set the threshold $R_G$ for them as $11.0$ and $8.0$, respectively.
The comparison results  are shown in Figure~\ref{fig:unconstrained_GB}, which show that algorithm was successful in significantly improving PPO under such unconstrained settings.
In fact,  by removing the constraint, our algorithm was able to focus solely on maximizing the reward without worrying about the  costs associated with its actions. As a result, it could explore more freely and achieve better results in challenging scenarios.

Overall, this experiment showcases the efficacy of our approach in enhancing the performance of an unconstrained agent, especially in tasks where achieving high returns is challenging.

\subsection{Conditional Value at Risk}
\begin{figure}[htbp]
\centering
\rotatebox[origin=c]{90}{\centering Return}
\showReturn[0.3]{0}{0}{0}{0}{Images/SafetyPointGoal1-v0/CVaR_return}{SafetyPointGoal}
\showReturn[0.3]{0}{0}{0}{0}{Images/SafetyCarGoal1-v0/CVaR_return}{SafetyCarGoal}

\rotatebox[origin=c]{90}{\centering Cost}
\showCost[0.3]{0}{0}{0}{0}{Images/SafetyPointGoal1-v0/CVaR_cost}
\showCost[0.3]{0}{0}{0}{0}{Images/SafetyCarGoal1-v0/CVaR_cost}

\showLegend[0.6]{10}{10}{15}{15}{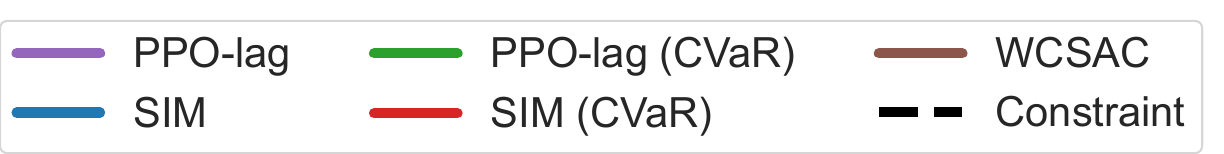}
    \caption{Comparison results with CVaR constraints.} 
    \label{fig:cvar_exp}
\end{figure}
So far, we have focused on expected cost constraints. In this section, we expand our experiments to CVaR constraints~\cite{yang2021wcsac} (to address \textbf{Q8}  - would our approach work with CVaR constrained problems?). We implemented a SIM version that works with CVaR constraints by using the  following criteria to select good trajectories: $R(\tau) \geq R_G$ and $C(\tau)+\alpha^{-1}\phi(\Phi^{-1}(\alpha))\sigma(C) \leq c_{max}$. Here, $\alpha=0.5$ represents the risk level, $\phi$ and $\Phi$ denote the probability density function (PDF) and cumulative distribution function (CDF) of the standard normal distribution, respectively, and $\sigma(C)$ is the standard deviation of the cost of the collected trajectories. We compare our approach with WC-SAC ~\cite{yang2021wcsac} (a state-of-the-art CVaR constrained RL algorithm). 
Additionally, we implement a PPO version with CVaR constrained (denoted as PPO-CVaR) for the sake of comparison.

The comparison results are shown in Figure~\ref{fig:cvar_exp}. Interestingly, the original version of the WC-SAC struggled to achieve satisfactory results. However,  PPO-CVaR approach performed exceptionally well in both environments, achieving improved performance while still maintaining lower costs than PPO-Lagrangian.
Furthermore, our algorithm SIM (CVaR) outperformed all other curves, achieving the same expected cost while offering even higher expected rewards. This indicates the superior performance and effectiveness of our proposed approach compared to the other baseline methods considered.
Overall, our experiments demonstrate that incorporating CVaR and SIM significantly enhances the performance of prior algorithms.

\subsection{Number of initial expert demonstrations}
{In this section, we aim to address the impact of the number of initial expert trajectories on our final performance \textbf{Q9}. To this end, we run the experiments with different numbers of expert trajectories, taken from the set $[5,10,25,50,100,300,500]$. The detailed results are shown in Figure~\textbf{\ref{fig:fig_num_initial}}.}
{Here, it is easy to observe that having a large number of expert demonstrations can offer better performance. This is possibly because having this high number of expert demonstrations  can reduce the number of explorations in the environment and quickly understand the criteria for classifying good and bad trajectories.}
\begin{figure}[htbp]
\centering
\rotatebox[origin=c]{90}{\centering Return}
\showReturn[0.3]{0}{0}{0}{0}{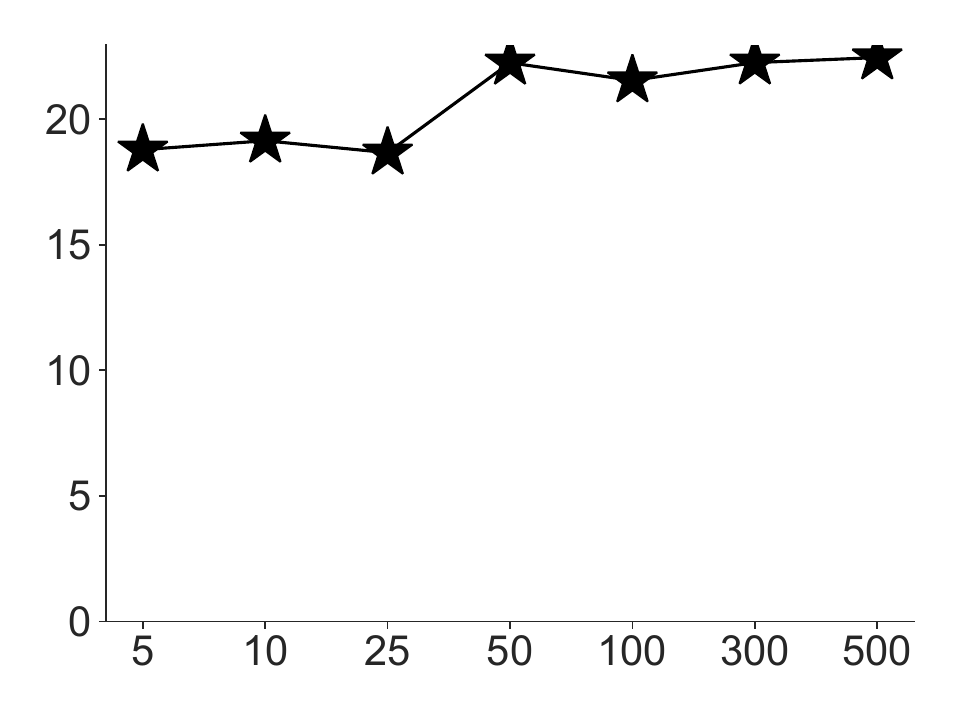}{SafetyPointGoal}
\showReturn[0.3]{0}{0}{0}{0}{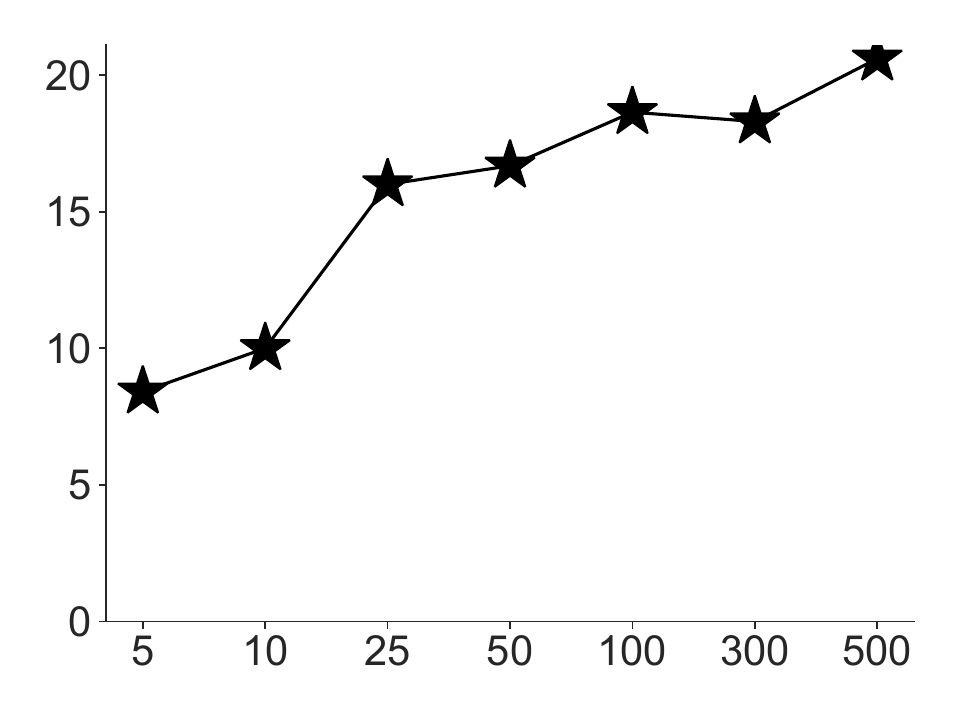}{SafetyCarGoal}

\rotatebox[origin=c]{90}{\centering Cost}
\showCost[0.3]{0}{0}{0}{0}{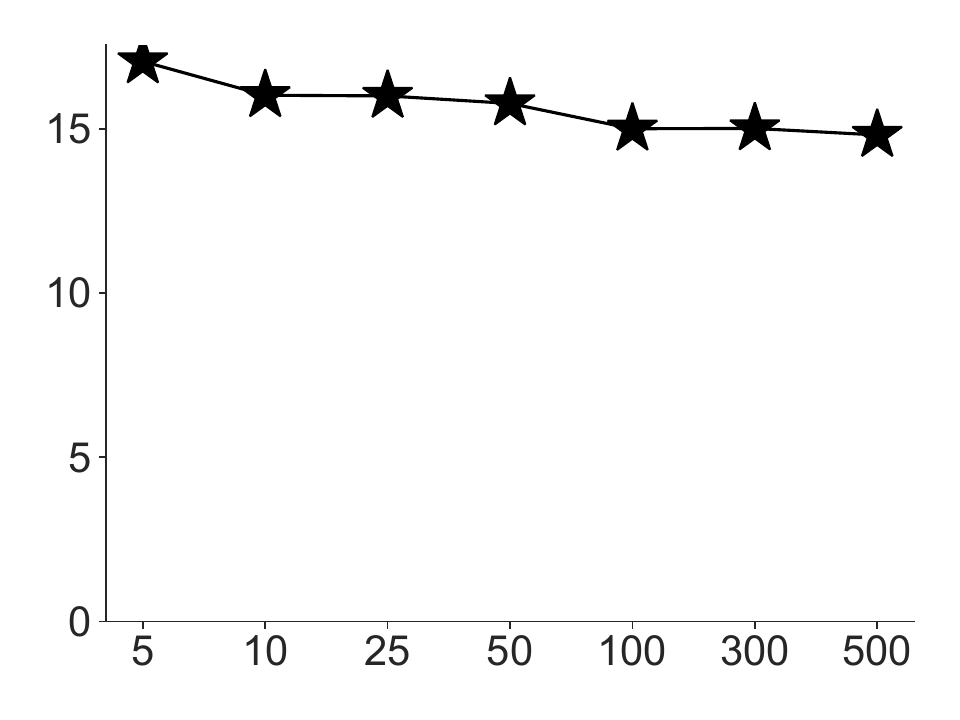}
\showCost[0.3]{0}{0}{0}{0}{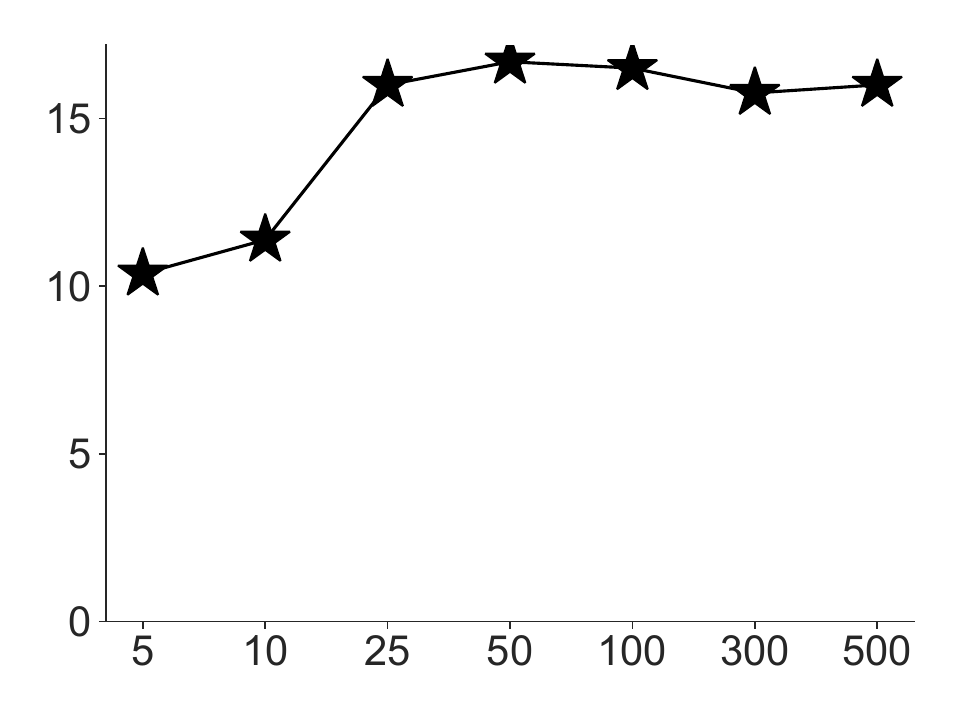}
    \caption{Best performance of the 7 different number of expert trajectories.}
    \label{fig:fig_num_initial}
\end{figure}
\subsection{Low-quality Initial Policy}

In practical scenarios, there is no guarantee that an initial policy would be able to generate a sufficient set of good trajectories. In particular, a low-quality initial policy would even struggle with exploring good actions. To showcase such a situation, we run our SIM with six different random initial policies and plot their training curves in Figure~\ref{fig:SIM_SafetyPointPush1-v0} clearly shows that SIM was unable to achieve high rewards for Seed \#3. This problem have raised a question 
\begin{figure}[htbp]
\centering
\showReturn[0.3]{10}{10}{10}{10}{Images/SafetyPointPush1-v0/GB_problem_return}{Return}
\showReturn[0.3]{10}{10}{10}{10}{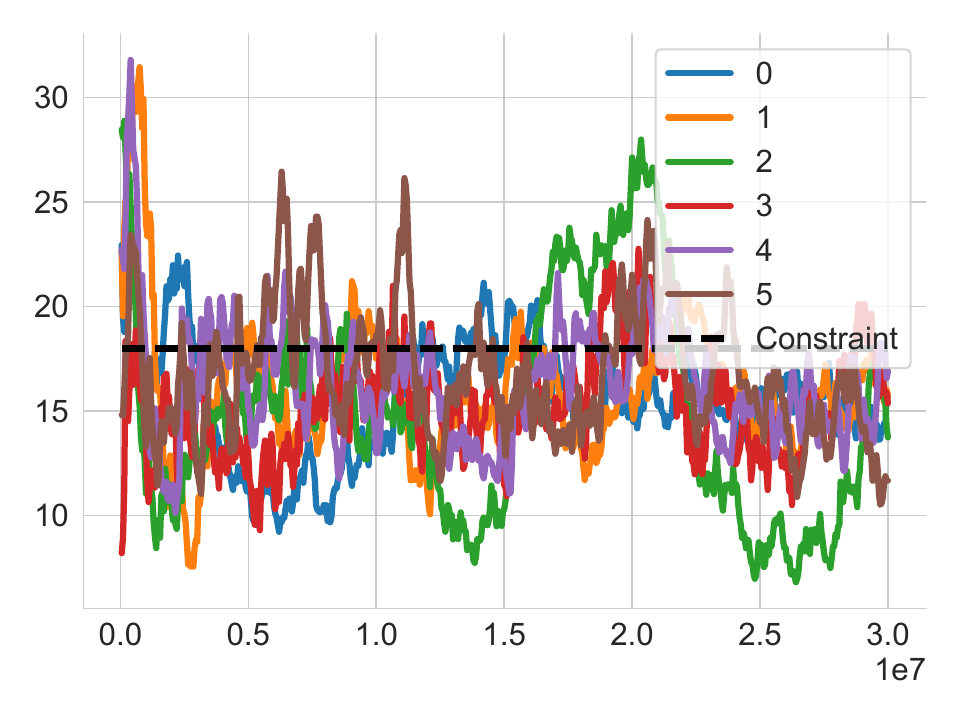}{Cost}
    \caption{Results of 6 different seeds of SIM in SafetyPointPush.}
    \label{fig:SIM_SafetyPointPush1-v0}
\end{figure}

Taking the above into consideration, we will show below that the issue can be addressed by using dynamic thresholds $R^G$. Our approach is to dynamically adjust the ``good'' threshold $R_G$ during each update step to incentivize the policy to perform at par with the highest-ranking of return trajectories within the collected trajectory set $T$: $R_G = \bbE_{\tau\sim T}[R(\tau)] + 2 \sigma_{\tau\sim T}[R(\tau)]$. 
The comparison of  the original SIM and SIM with dynamic $R_G$ for Seed \#3 is shown in  Figure~\ref{fig:dynamic_improvement}, which clearly indicates the superiority of the dynamic SIM, compared to the static version (as well as the PPO baseline). 
Notably, it's essential to acknowledge that due to reducing the threshold $R_G$, the training can achieve higher rewards when generating good trajectories is challenging. However, when the highest $R_G$ is attainable, it can not replicate the performance of fixed $R_G$.


\begin{figure}[htbp]
\centering
\showReturn[0.3]{0}{0}{0}{0}{Images/SafetyPointPush1-v0/dynamic_return}{Return}
\showReturn[0.3]{0}{0}{0}{0}{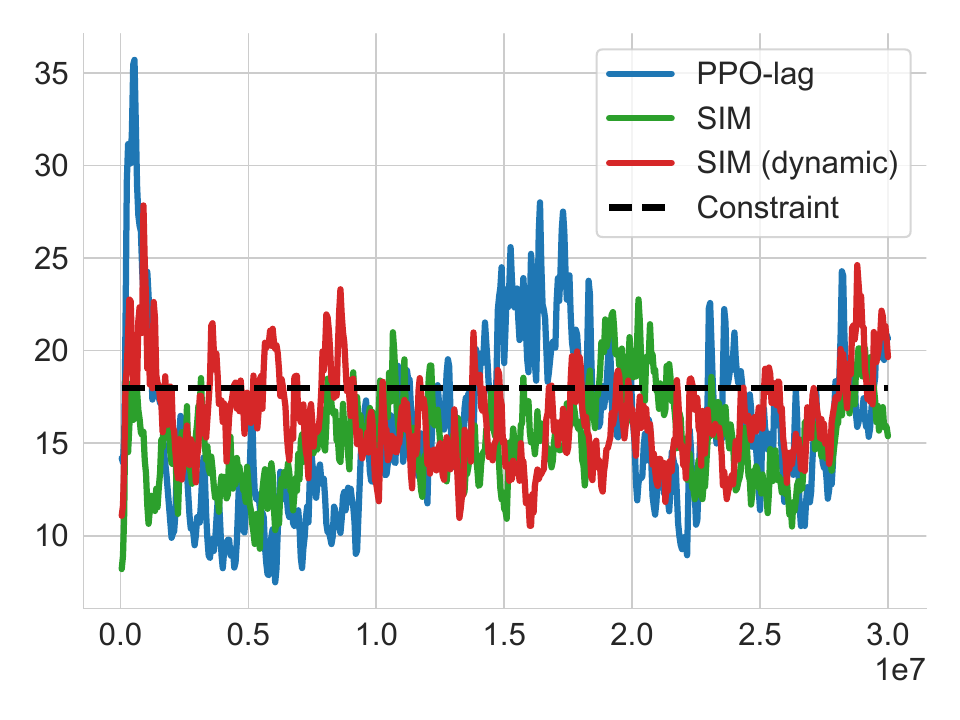}{Cost}
    \caption{dynamic experiment in SafetyPointPush.}
    \label{fig:dynamic_improvement}
\end{figure}

\subsection{Experiments on Mujoco domains}
\subsubsection{Mujoco Circle}
In this section, we present additional comparisons on Mujoco-Circle environments~\cite{achiam2017constrained,Safety-Gymnasium}.  Figure~\ref{fig:full_circle_result} shows our comparison results, which clearly compare the strength of SIM over other baseline methods.

\begin{figure}[htbp]
\centering
\rotatebox[origin=c]{90}{\centering Return}
\showReturn[0.3]{0}{0}{0}{0}{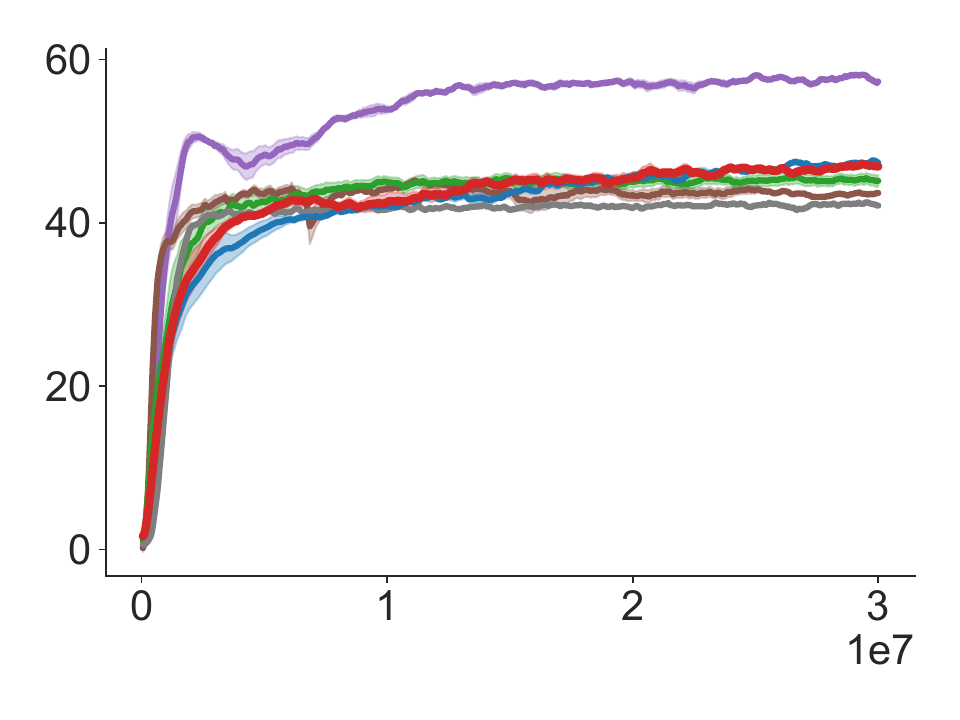}{SafetyPointCircle}
\showReturn[0.3]{0}{0}{0}{0}{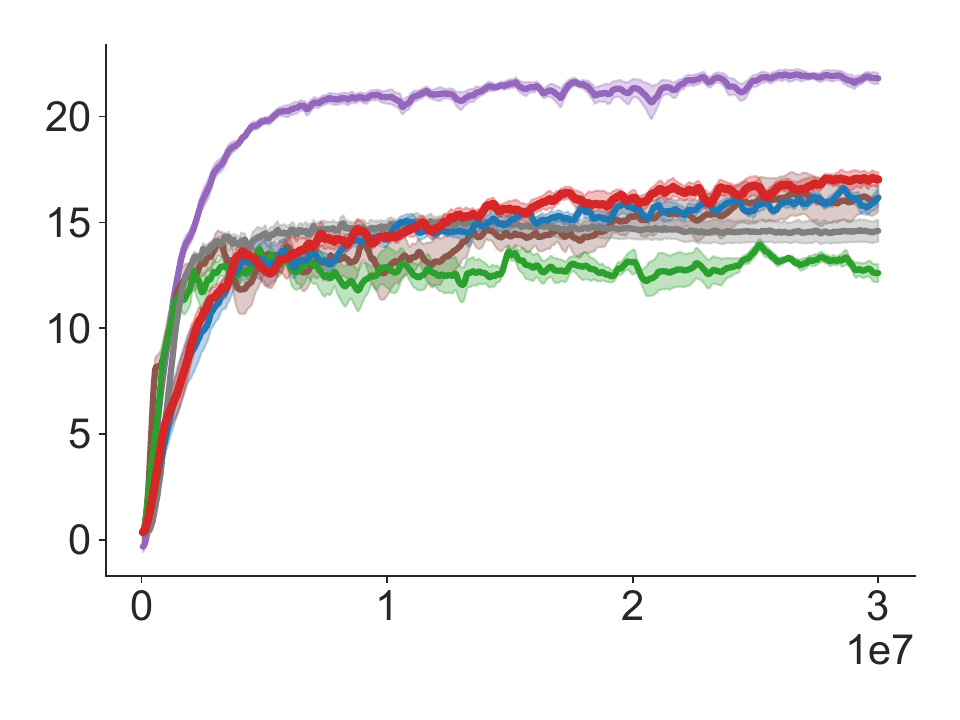}{SafetyCarCircle}

\rotatebox[origin=c]{90}{\centering Cost}
\showCost[0.3]{0}{0}{0}{0}{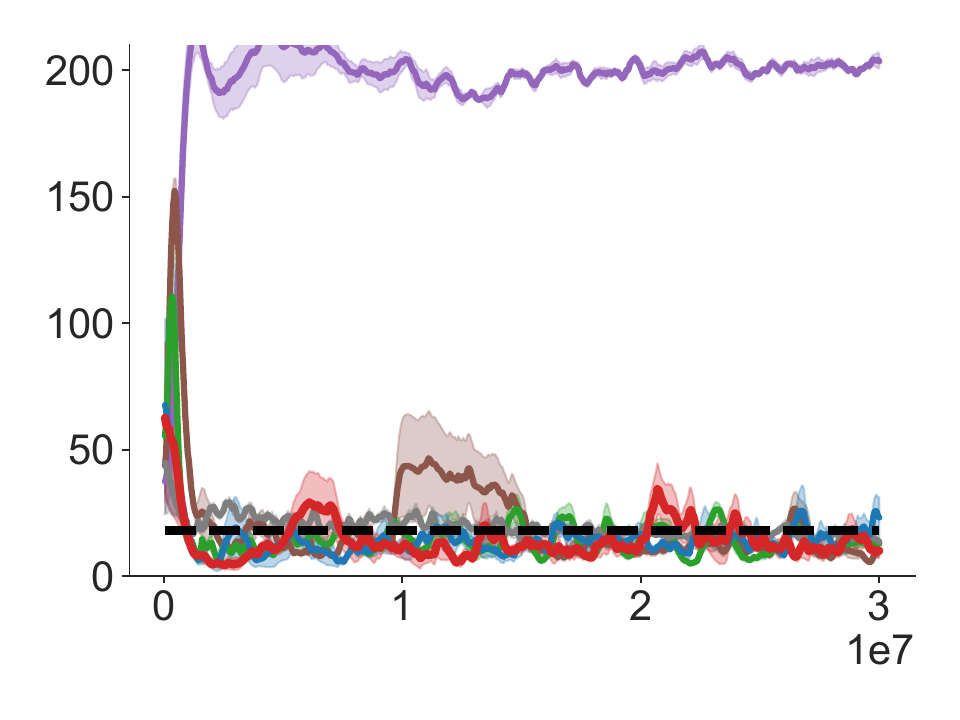}
\showCost[0.3]{0}{0}{0}{0}{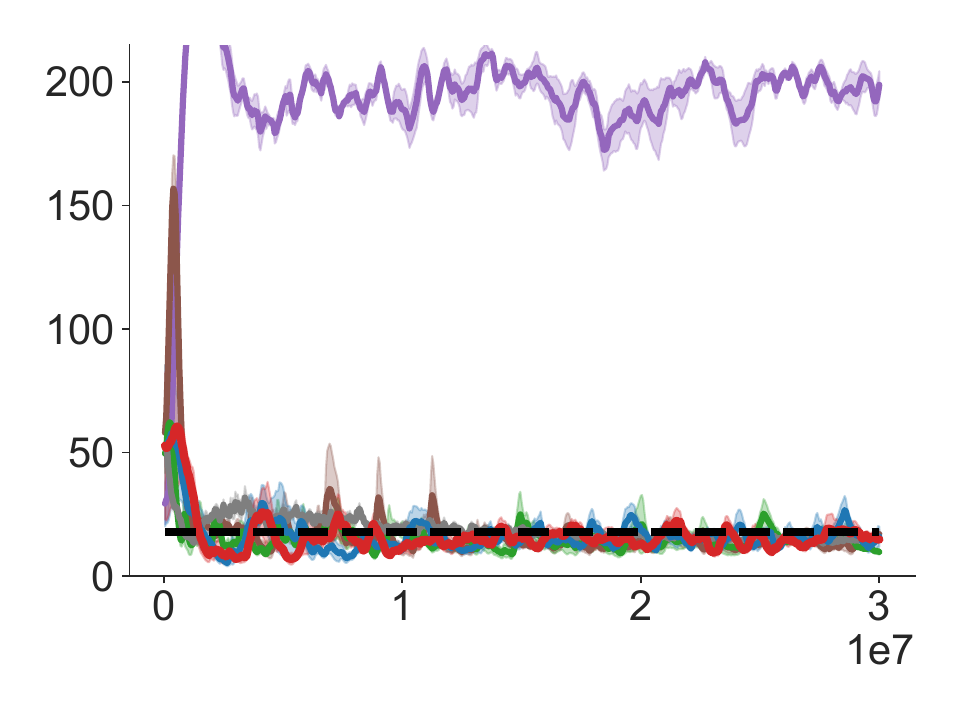}
\showLegend[0.5]{10}{10}{15}{15}{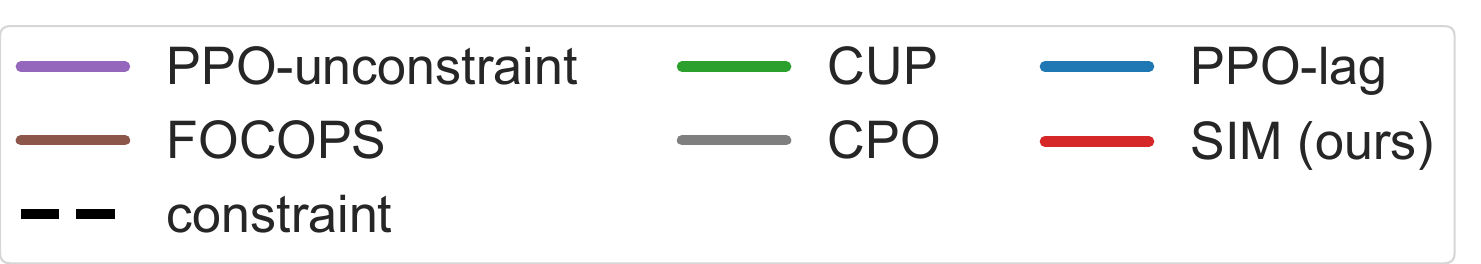}
    \caption{Results for \textit{Mujoco Circle} environments.}
    \label{fig:full_circle_result}
\end{figure}

\subsubsection{Mujoco Velocity}
We further provide experiments on Mujoco-Velocity environments. The performance curves reported in  Figure~\ref{fig:full_velocity_result} shows our comparison results, which clearly demonstrate the superiority of SIM over other baseline methods in satisfying the constraint during the training as well as having a high return. 

\begin{figure}[htbp]
\centering
\rotatebox[origin=c]{90}{\centering Return}
\showReturn[0.3]{0}{0}{0}{0}{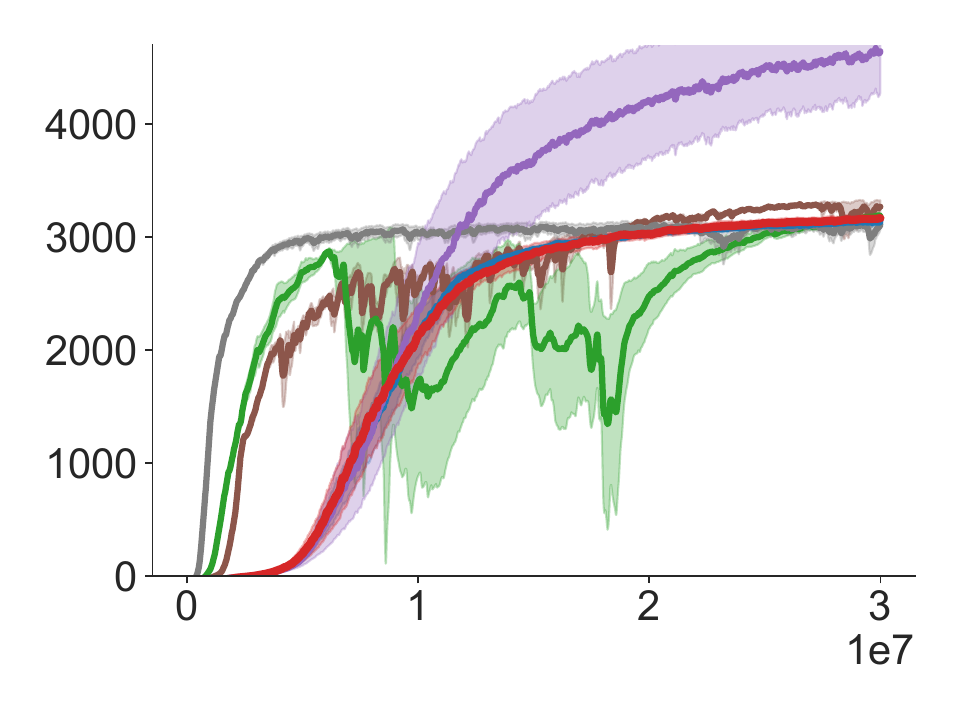}{SafetyAntVelocity}
\showReturn[0.3]{0}{0}{0}{0}{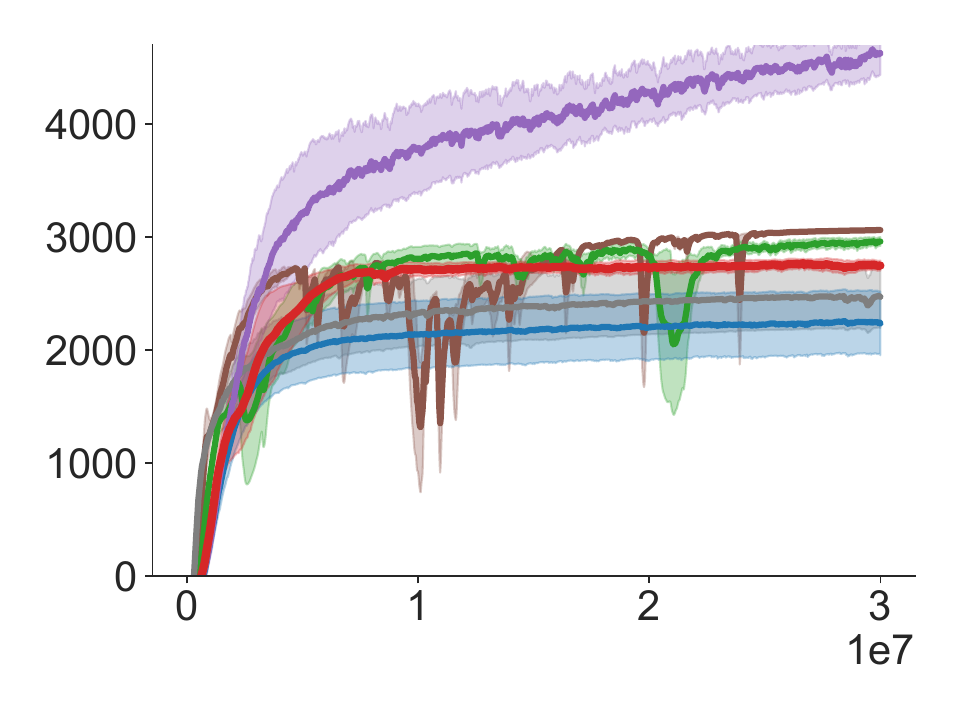}{SafetyHalfCheetahVelocity}

\rotatebox[origin=c]{90}{\centering Cost}
\showCost[0.3]{0}{0}{0}{0}{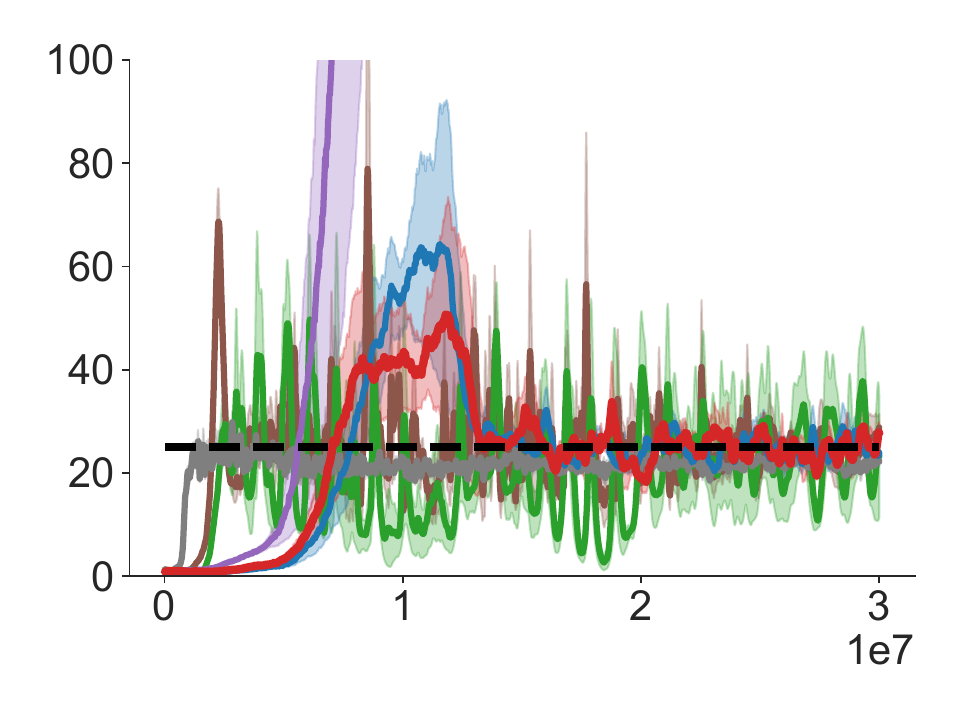}
\showCost[0.3]{0}{0}{0}{0}{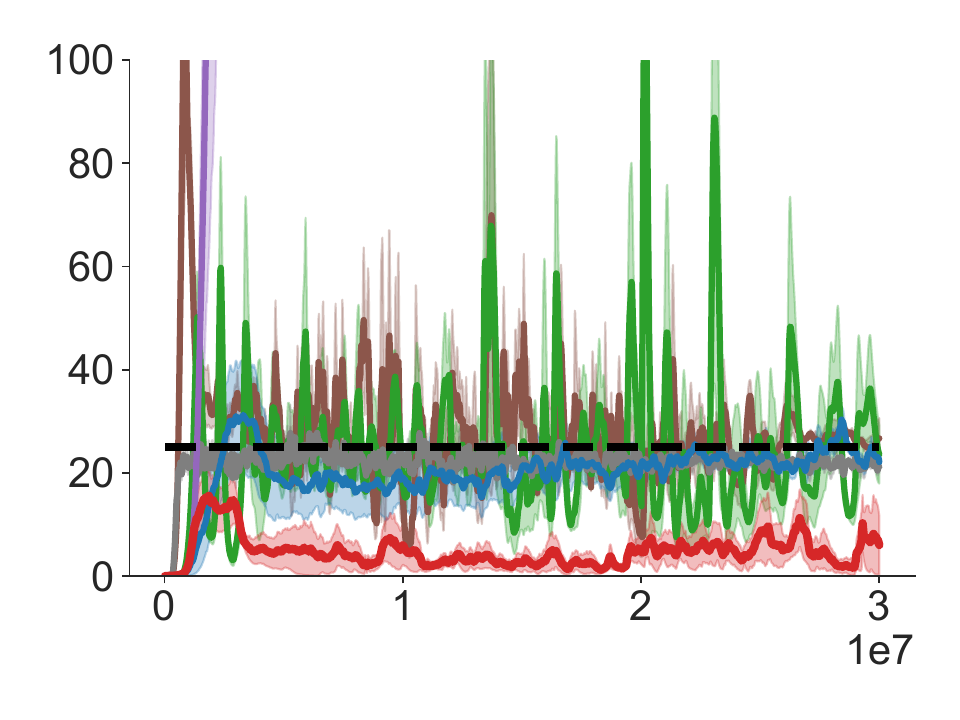}
\showLegend[0.5]{10}{10}{15}{15}{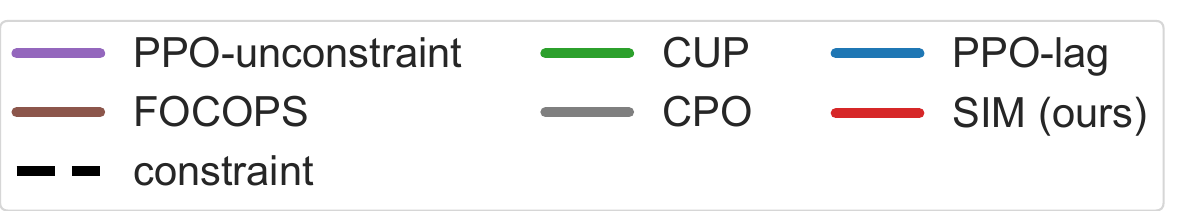}
    \caption{Results for \textit{Mujoco Velocity} environments.}
    \label{fig:full_velocity_result}
\end{figure}

\end{document}